\documentclass[lettersize,journal]{IEEEtran}
\usepackage{amsmath,amsfonts}
\usepackage{svg}
\usepackage[ruled]{algorithm2e}
\usepackage{stackengine}
\usepackage{array}
\usepackage[caption=false,font=normalsize,labelfont=sf,textfont=sf]{subfig}
\usepackage{textcomp}
\usepackage{stfloats}
\usepackage{url}
\usepackage{verbatim}
\usepackage{graphicx}
\usepackage{cite}
\usepackage{xcolor}
\usepackage{adjustbox}
\usepackage{hyperref}
\usepackage{multirow}
\def\delequal{\mathrel{\ensurestackMath{\stackon[1pt]{=}{\scriptstyle\Delta}}}}

\begin{document}

\title{Prioritized Experience-based Reinforcement Learning with Human Guidance for Autonomous Driving}

\author{Jingda Wu, \IEEEmembership{Student Member,~IEEE}
Zhiyu Huang, \IEEEmembership{Student Member,~IEEE}
Wenhui Huang, and
Chen Lv, \IEEEmembership{Senior Member,~IEEE}
\thanks{J. Wu, Z. Huang, W. Huang and C. Lv are with the School of Mechanical and Aerospace Engineering, Nanyang Technological University, Singapore, 639798. (E-mail: \{jingda001, zhiyu001, wenhui001\}@e.ntu.edu.sg,  lyuchen@ntu.edu.sg)}
\thanks{Corresponding author: Chen Lv}
\thanks{This paper has been published in \textit{IEEE Transactions on Neural Networks and Learning Systems}. DOI: 10.1109/TNNLS.2022.3177685}
\thanks{The code associated with this paper is available at this \href{https://github.com/wujingda/Prioritized-Human-in-the-loop-End-to-end-Autonomous-Driving}{link} }
}

\markboth{Journal of \LaTeX\ Class Files,~Vol.~14, No.~8, August~2021}%
{Shell \MakeLowercase{\textit{et al.}}: A Sample Article Using IEEEtran.cls for IEEE Journals}


\maketitle

\begin{abstract}
Reinforcement learning (RL) requires skillful definition and remarkable computational efforts to solve optimization and control problems, which could impair its prospect. Introducing human guidance into reinforcement learning is a promising way to improve learning performance. In this paper, a comprehensive human guidance-based reinforcement learning framework is established. A novel prioritized experience replay mechanism that adapts to human guidance in the reinforcement learning process is proposed to boost the efficiency and performance of the reinforcement learning algorithm. To relieve the heavy workload on human participants, a behavior model is established based on an incremental online learning method to mimic human actions. We design two challenging autonomous driving tasks for evaluating the proposed algorithm. Experiments are conducted to access the training and testing performance and learning mechanism of the proposed algorithm. Comparative results against the state-of-the-art methods suggest the advantages of our algorithm in terms of learning efficiency, performance, and robustness.
\end{abstract}


\section{Introduction}
\IEEEPARstart{R}{einforcement} learning (RL) has substantially contributed to numerous fields \cite{RN1,RN2,RN3,RN4} by solving control and optimization problems. As a branch of machine learning methods, RL improves the capability of controlling agents in black-box environments through the exploratory trial-and-error principle \cite{RN5}. Recent popular RL algorithms, e.g., rainbow deep Q-learning \cite{RN6}, proximal policy optimization (PPO) \cite{RN7}, and soft actor-critic (SAC) \cite{RN8}, have shown ability in handling high-dimensional environment representation and generalization, due to the introduction of deep neural networks. Albeit RL can achieve good performance in complex tasks, its drawback emerges that their interactions with the environment are very inefficient \cite{RN9}. Thus, using RL to solve a problem needs skillful definitions and settings and consumes remarkable computational resources \cite{RN10}. 

Combining human guidance with RL can be a promising way to mitigate the above drawback \cite{RN11}. First, human intervention has been used to improve RL performance. Intervention is triggered by unfavorable actions and should be avoided by RL. Then, the human demonstration is a powerful tool to enhance RL’s ability \cite{RN12}. In this context, the objective functions are generally reshaped compatible with supervised learning to improve efficiency \cite{RN13}. 
Despite the above human guidance-based methods, RL needs to process numerous data from its self-explorations. The existing methods do not particularly optimize the utilization of human guidance data; consequently, they still need great human workloads to avoid submersion of guidance in exploratory data. Additionally, human guidance, which is variant to proficiency, mental and physical status of participants, should not be equally treated since some low-quality guidance can even impair the RL performance.

We propose a priority-based experience replay method on human guidance and put forward the associated human guidance-based RL algorithm to bridge the abovementioned gap. Our approach is off-policy, which leverages the experience replay mechanism \cite{RN14} to maximize the utilization efficiency of self-exploratory data. The proposed priority replay mechanism can further improve the utilization efficiency of human guidance data by quantifying their values and weighing their utilized probability, which ultimately augments the RL performance. As a result, the efficiency can be improved by over seven times under the adopted task. The schematic diagram of our algorithm is depicted in Fig. \ref{Figure1}. To evaluate the training and testing performance of our proposed method, we design two challenging autonomous driving scenarios. The experimental results suggest the advance of the proposed algorithm compared to state-of-the-art baselines in learning efficiency, practical performance, and robustness.

\begin{figure}[tb]
\begin{center}
\noindent
  \includegraphics[width=\linewidth]{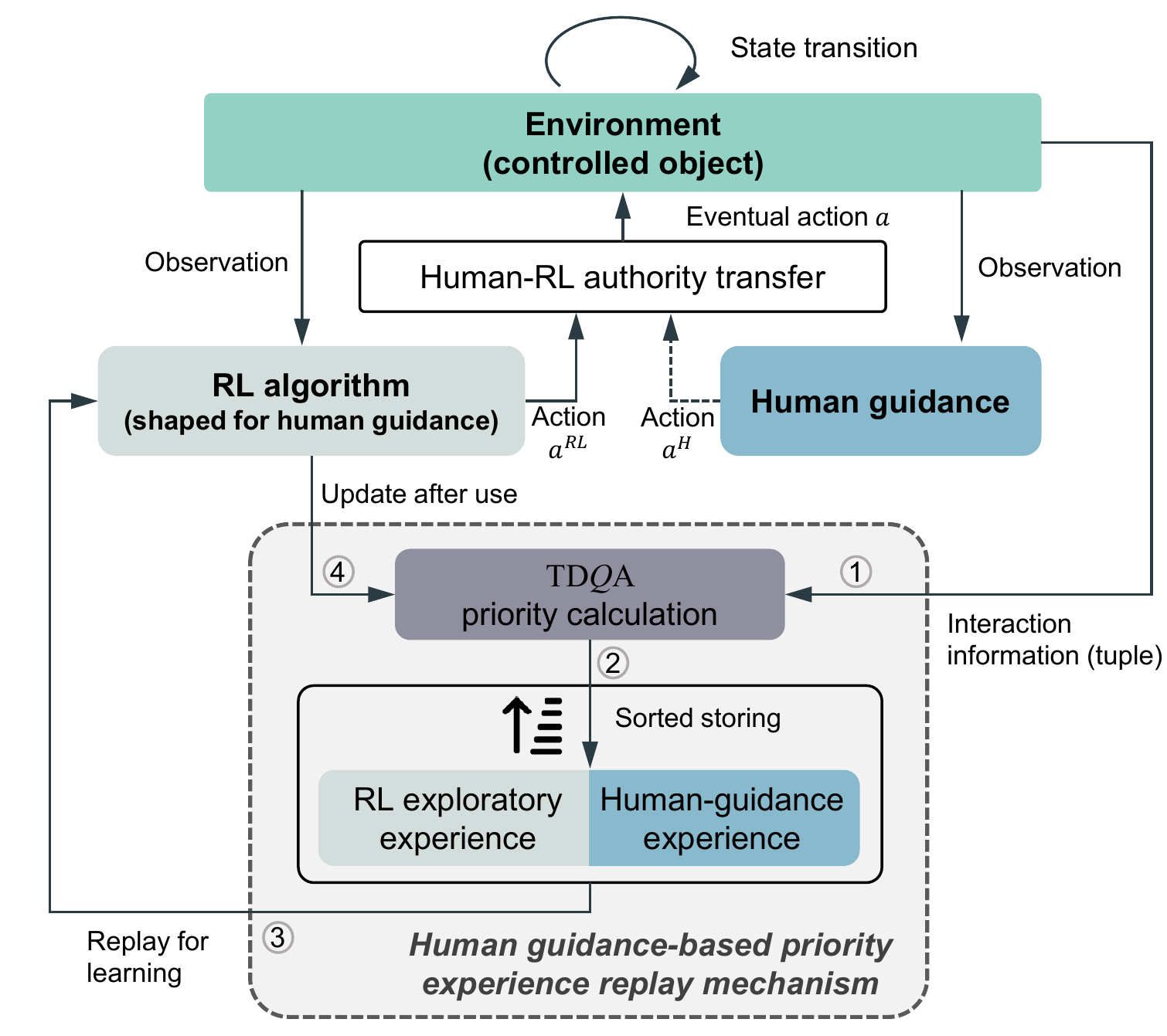}
  \end{center}
    \caption[Framework of the proposed human-guided RL.]{Framework of the proposed human-guided reinforcement learning. The RL algorithm in this report is shaped in multiple aspects to adapt to human guidance. In the proposed human guidance-based priority experience replay mechanism, TDQA represents the proposed priority calculation scheme, and the number 1-4 indicates the flow sequence of data. The dotted line of the action signal represents that the framework allows intermittent human-in-the-loop guidance. }\label{Figure1}
\end{figure}

The contribution of this report can be summarized into three aspects. 1) We propose a novel prioritized experience utilization mechanism regarding human guidance in the RL process to improve performance. 2) We establish a comprehensive and holistic framework of human guidance-based RL by integrating the human-RL action switch scheme, behavior cloning-based objective function, human-demonstration replay method, and human-intervention reward shaping mechanism.3) We validate the superior performance of the proposed algorithm in solving challenging autonomous driving tasks comprehensively.

The remainder is organized as follows: a review of related work is provided in Section \ref{section2}, preliminaries for the proposed algorithm is introduced in Section \ref{section3}, Section \ref{section4} provides the proposed human guidance-based reinforcement learning algorithm, a human behavior model for substituting real human participant is established in Section \ref{section5}. Section \ref{section6} presents the problem formulation for the adopted autonomous driving tasks, Section \ref{section7} provides the experimental results, and the conclusion is drawn in Section \ref{section8}.

\section{Related Work}\label{section2}

Sample efficiency bottlenecks the training and performance of RL. Combining human guidance with RL is a promising way to mitigate the challenge. Three categories of human guidance have been integrated into RL. 

The first one is human feedback, where the human expert’s prior knowledge about the task could be used to qualitatively or quantitatively score the RL behaviors \cite{RN15}. In this manner, an RL-based unmanned ground vehicle was guided to run through a maze \cite{RN16}. However, the feedback is high-demanding on human ability and thus is no longer popular in recent studies. 

The second branch is human intervention. Intervention is a more direct manifestation of human knowledge than giving feedback. RL agents are devised to reduce their confidence in adopted actions if intervention occurs \cite{RN17}. \cite{RN18} employed real humans to detect catastrophic actions of DQN in playing Atari games, where humans were required to intervene in the training process to block the risk. It punished the human intervened scenes through the reward-shaping technique to prevent RL from reaching the unfavorable situations again. With a similar idea, \cite{RN19} devised a reward shaping-based PPO algorithm and made the RL agent complete the drone driving tasks under human interventions. In this report, the above-mentioned reward shaping scheme is also adopted, and more importantly, we provide a theoretical derivation and related discussion on the optimality of the human intervention-based reward shaping method.

The human demonstration is the other way to enhance RL performance. For discrete-action RL, the DQfD algorithm \cite{RN12} shaped the value function of DQN using human demonstration. \cite{RN20} presented a double experience buffer setting to separately store the RL data and human demonstrations. For more complicated RL with actor-critic architecture, the policy function is usually modified to be compatible with learning from demonstration. The behavior cloning objective has been added to the objective of the policy function to greatly improve learning efficiency, which is a milestone in the field. In this way, dexterous manipulations of high degree-of-freedom robotic arms \cite{RN21,RN22,RN23} and human-level game operation \cite{RN17} were achieved based on the state-of-the-art RL algorithms. In this report, the behavior cloning objective and its associated human guidance-based actor-critic framework is also integrated into our method. However, it is not reasonable for equal treatment on various demonstrations, which is adopted in existing methods. First, without optimizing the utilization, small-scale human demonstrations would be submerged in the numerous RL-generated data. Second, human guidance is variant due to the proficiency and status of participants, and some low-quality guidance can even impair RL performance. Noticeably, these drawbacks are to be overcome by the proposed prioritized experience utilization mechanism.

\section{Preliminaries}\label{section3}

In this section, we first introduce the notation and concept of off-policy actor-critic RL, and we then illustrate the prioritized experience replay mechanism. All three parts in this section are the base for the proposed human-guidance-based RL algorithm.

\subsection{Notation}\label{section3.1}

We consider a standard RL setting where an RL agent interacts with the controlled environment. Such an interaction can be formulated as a discrete-time Markov decision process (MDP), defined by the tuple \((\mathcal{S},\mathcal{A},R,p)\). The state-space \(\mathcal{S}\) consists of continuous state variables \(\mathbf{s}\) and the action space constitutes continuous action variables \(\mathbf{a}\). \(R(\cdot\vert \mathbf{s},\mathbf{a}):\mathcal{S}\times \mathcal{A}\to r\) is a reward function mapping the state-action pair \((\mathbf{s},\mathbf{a})\) to a deterministic reward value \(r\). The environment dynamics generates state transition probability \(p(\cdot\vert \mathbf{s},\mathbf{a}):\mathcal{S}\times \mathcal{A}\to P(\mathbf{s}^{\prime})\) mapping the state-action pair \((\mathbf{s},\mathbf{a})\) to the probability distribution over the next state \(\mathbf{s}^{\prime}\).

At each time step \(t\), the agent observes the state \(\mathbf{s}_t\in \mathcal{S}\) and sends the action \(\mathbf{a}_t\in \mathcal{A}\) to the environment, receiving the feedback of a scalar reward \(r_t\) and next state \(\mathbf{s}_{t+1}\). The agent’s behavior is determined by a policy \(\pi(\mathbf{a}_t\vert \mathbf{s}_t ):\mathcal{S}\to P(\mathbf{a}_t)\), which maps a state to the probability distribution over candidate actions. We utilize \(\rho_{\pi}\) to represent the state-action distribution induced by the policy \(\pi\).

\subsection{Off-policy Actor-critic Architecture}\label{section3.2}

The goal of RL is to optimize the policy which maximizes the expected value \(\mathcal{V}\) over the environment dynamics. A Bellman value function (also called critic) is established to estimate \(\mathcal{V}\) in a bootstrapping way. This value function is usually called \(Q\). Under an arbitrary policy \(\pi\), \(Q\) is defined as:
\begin{align}\label{eq1}
    Q^\pi (\mathbf{s}_t,\mathbf{a}_t)=r_t+\gamma\underset{(\mathbf{s}_{t+1},\mathbf{a}_{t+1})\sim\rho_\pi}{\mathbb{E}}[Q^\pi(\mathbf{s}_{t+1},\mathbf{a}_{t+1})],
\end{align}
where \(\gamma\in(0,1)\) is the discount factor. Then the policy function (also called actor) can be obtained concerning maximized \(Q\), represented as:
\begin{equation}\label{eq2}
    \pi=\arg\max_\pi\left[\underset{(\mathbf{s},\mathbf{a})\sim\rho_\pi}{\mathbb{E}} \left[Q^\pi\ (\mathbf{s},\mathbf{a})\right]\right],
\end{equation}

In practice, value function pursues the evaluation regarding only the optimal policy \(\pi^\star\), regardless of the policy executing the interaction. Therefore, RL decouples the policy evaluation process and the policy’s behavior, which makes the agent update in an off-policy manner. 

We use neural networks as the function approximator to formulate the actor and critic, the objectives are then reached through the loss functions. Specifically, the loss function of the critic \(\mathcal{L}^Q\), and the actor \(\mathcal{L}^\pi\) can be expressed as:
\begin{equation}\label{eq3}
    \mathcal{L}^Q\left(\theta\right)=r_t+\gamma\mathbb{E}\left[Q\left(\mathbf{s}_{t+1},\pi\left(\mathbf{s}_{t+1};\phi\right);\theta\right)\right]-Q\left(\mathbf{s}_t,\mathbf{a}_t;\theta\right),
\end{equation}
\begin{equation}\label{eq4}
    \mathcal{L}^\pi\ \left(\phi\right)=-Q\left(\mathbf{s}_t,\pi\left(\cdot\vert \mathbf{s}_t;\phi\right);\theta\right),
\end{equation}
where \(Q(\cdot;\theta)\) represents the parameterized critic function and \(\theta\) represents the parameters of the critic network, \(\pi(\cdot;\phi)\) represents the parameterized actor function and \(\phi\) represents the parameters of the actor network. Hereinafter, the parameter \(\theta\) and \(\phi\) can be omitted if no ambiguity exists.

\subsection{Prioritized Experience Replay Mechanism}\label{section3.3}
The experience replay mechanism establishes an experience buffer to store the data at each interaction. Accordingly, the RL agent can retrieve data generated by previous policies from the buffer for policy evaluation and improvement.

Given an arbitrary time step \(t\), the interaction between the RL agent and the environment generates a transition tuple, which is stored into the experience replay buffer as:
\begin{equation}\label{eq5}
    \mathcal{B}\gets\zeta_t=(\mathbf{s}_t,\mathbf{a}_t,r_t,\mathbf{s}_{t+1}).
\end{equation}

Conventionally, the experience in the buffer is retrieved from the buffer using uniform random sampling. In a more efficient method, prioritized experience replay mechanism (PER) \cite{RN24}, the data sample is subjected to a nonuniform distribution \(\mathcal{I}\), and its probability mass function \(p_\mathcal{I}\sim\mathcal{I}\) can be expressed as:
\begin{equation}\label{eq6}
    p_\mathcal{I}\ (i)=\frac{\mathbf{p}_{i}^{\alpha}}{\sum_k \mathbf{p}_{k}^{\alpha} },
\end{equation}
where \(\alpha\in[0,1]\) is the scaling coefficient, \(\mathbf{p}\) represents the priority of each tuple \(i\), which is determined by the temporal difference (\(TD\)) error \(\delta^{TD}\) and expressed as:
\begin{align}\label{eq7}
\begin{split}
        \mathbf{p}_i&=\vert \delta_{i}^{TD} \vert +\varepsilon \\
        &=\vert r_i+\gamma\cdot Q\left(\mathbf{s}_{i+1},\pi\left(\cdot\vert\mathbf{s}_{i+1};\phi\right);\theta\right)-Q\left(\mathbf{s}_i,\mathbf{a}_i;\theta\right)\vert+\varepsilon,
\end{split}
\end{align}
where \(\varepsilon\in\mathbb{R}^{+}\) is a small positive constant to guarantee the probability larger than zero. A larger \(TD\) error indicates an experience worth learning to a higher extent. Thus, the \(TD\) error-based prioritized experience replay mechanism can improve the RL training efficiency.

\section{Human-in-the-loop Reinforcement Learning}\label{section4}

In this section, we first summarize the human behaviors in the RL training process which can be leveraged in the algorithm design. Based on that, we establish an actor-critic framework adapting to human guidance. Then, two modules are proposed to further improve RL in the context of human guidance: a novel prioritized experience replay mechanism concerning human demonstration, and a reward shaping technique concerning human intervention. Finally, a holistic human-in-the-loop RL algorithm is instantiated using the above components. 

\subsection{Human Guidance Behavior in RL Training}\label{section4.1}
We define two useful human guidance behaviors in the RL training process: intervention and demonstration. 

\subsubsection*{Intervention}
Human participants recognize RL interaction scenes and identify whether a guidance behavior should be conducted based on their prior knowledge and reasoning abilities. If human participants decide to intervene, they can manipulate the equipment to get the control authority (partially or totally) from the RL agent. The intervention generally happens when the RL agent conducts catastrophic actions or is stuck in local optima traps. Thus, RL could learn to avoid unfavorable situations from the intervention.

\subsubsection*{Demonstration}
Human participants perform their actions when an intervention event happens, which generates the corresponded reward signal and next-step state. The generated transition tuple can be seen as a piece of demonstration data since it is induced by human policy instead of the RL’s behavior policy. RL algorithm could learn human behavior from the demonstration.

State-of-the-art human-guidance-based RL algorithms have been integrating learning from intervention (LfI) \cite{RN18}, and learning from demonstration (LfD) \cite{RN25}. In this report, both LfI and LfD will be employed in the proposed architecture. Specifically, LfI based on the reward shaping technique is utilized in the reward function definition, while LfD plays its role in the underlying principles of the algorithm.

\subsection{Human-guidance-based Actor-critic Framework}\label{section4.2}

In this section, we elaborate on the interaction mechanism and learning objective of the proposed human-guidance-based actor-critic RL algorithm.

First, we focus on the interaction mechanism. In the standard interaction between RL and environment, RL’s behavior policy will output actions to explore the environment. Given an off-policy actor-critic RL, the above process is shown as:
\begin{equation}\label{eq8}
    \mathbf{a}_{t}^{RL}=\pi(\cdot\vert\mathbf{s}_t;\phi)+\xi_a\odot\mathbf{a}_t^{\text{std}},
\end{equation}
where \({a}_t^{std}\in \mathbb{R}^{\text{dim}(\mathcal{A})} \) is a training-dependent variable that scales the exploration noise, \(\odot\) represents the Hadamard product and \(\xi_a\sim \mathcal{N}(0,\mathbf{I}^{\text{dim}\left(\mathcal{A}\right)} )\)

We give full authority to human participants whenever they decided to take control in the training loop of RL. Thus, the eventual action is filtered by a mask as:
\begin{equation}\label{eq9}
    \mathbf{a}_t=(\mathbf{I}^{dim(\mathcal{A})} -\mathbf{\Delta}_t )\cdot\mathbf{a}_t^{RL}+\mathbf{\Delta}_t\cdot\mathbf{a}_t^H,
\end{equation}
where \(\mathbf{a}_t^H\) represents the action from the human participant’s policy, \(\mathbf{\Delta}_t\in\mathbb{R}^{dim(\mathcal{A})}\) is a demonstration mask: it is an identity matrix when human demonstration happens and a zero matrix in the non-demonstrated step.

The interaction transition tuple \(\zeta\) will be recorded and stored into the experience replay buffer once the action is sent to the environment. In particular, actions from the human policy and the RL policy are stored in the same buffer. For this context, the new transition tuple \(\zeta\) is defined to discriminate human demonstrations from normal RL experiences as:

\begin{equation}\label{eq10}
    \zeta_i=(\mathbf{s}_i,\mathbf{a}_i,r_i,\mathbf{s}_{i+1},\mathbf{\Delta}_i).
\end{equation}
Then, we focus on the learning objective. Given a batch of transition tuples with batch size \(N\), there could exist data \(\zeta_{N_1}\) from the RL policy and \(\zeta_{N_2=N-N_1}\) from the human policy. The critic network, based on the optimal value function, can learn from both policies. Thus, its loss function is calculated as:
\begin{align}\label{eq11}
    \begin{split}
        &\mathcal{L}^Q (\theta)= \\
        &\frac{1}{N_1} \sum_i^{N_1} \Vert r_i+\gamma Q(\mathbf{s}_{i+1},\pi(\cdot\vert \mathbf{s}_{i+1});\theta)-Q(\mathbf{s}_i,\mathbf{a}_i^{RL};\theta)\Vert_2^2 \\
&+\frac{1}{N_2} \sum_j^{N_2}\Vert r_j+\gamma Q(\mathbf{s}_{j+1},\pi(\cdot\vert \mathbf{s}_{j+1});\theta)-Q(\mathbf{s}_j,\mathbf{a}_j^H;\theta)\Vert_2^2 .
    \end{split}
\end{align}

Given the data from the human policy, the actor should learn from these demonstrations in addition to maximizing the critic’s value. Hence, we devise the loss function of the actor network considering behavior cloning as:
\begin{equation}\label{eq12}
\begin{split}
    \mathcal{L}^\pi\ (\phi)&=
\frac{1}{N_1} \sum_i^{N_1}\left[-Q(\mathbf{s}_i,\pi(\cdot\vert \mathbf{s}_i;\phi);\theta)\right]\\
&+\frac{1}{N_2}  \sum_j^{N_2}[\omega\cdot \Vert \mathbf{a}_j^H-\pi(\cdot\vert \mathbf{s}_j;\phi)\Vert_2^2],
\end{split}
\end{equation}
where \(\omega\) is a manually determined constant that weighs the importance of behavior cloning.  

It is noticeable that the mean squared error (MSE) losses involved in the above formulas are for exemplified calculation, meaning that they can be alternated by any loss functions.

\subsection{Prioritized Human Demonstration Replay}\label{section4.3}
In this section, we put forward a novel PER mechanism for human demonstration. 

Human demonstrations are generally more critical than most exploration from RL’s behavior policy due to prior knowledge and reasoning ability. Thus, a more effective method is needed to weigh human demonstrations among the buffer. We propose an advantage-based metric instead of TD-error of the normal PER to establish the prioritized replay mechanism.

First, we define an advantage measure regarding the human demonstration against the RL’s behavior policy. Since the critic, i.e., value function, can evaluate the policy, we calculate the difference between the \(Q\) value of the human action and that of the RL action. Given a human-demonstration transition tuple \((\mathbf{s}_i,\mathbf{a}_i=\mathbf{a}_i^H,r_i,\mathbf{s}_{i+1})\), the priority level \(\mathbf{p}\) is defined as:
\begin{equation}\label{eq13}
    \mathbf{p}_i\delequal \vert \delta_i^{TD}\vert +\varepsilon+\exp \left[Q(\mathbf{s}_i,\mathbf{a}_i^H;\theta)-Q(\mathbf{s}_i,\pi(\cdot\vert \mathbf{s}_i);\theta)\right],
\end{equation}
where \(\exp\) is the exponential function to guarantee the non-negative advantage value.

We call the last term of the Eq. \ref{eq13} the \(Q\)-advantage term, which evaluates to what extent should a specific human-demonstration tuple be retrieved except the \(TD\)-error metric. Through the RL training process, the RL agent’s ability varies and the priority level of one human-demonstration tuple changes accordingly, which gives rise to a dynamic priority mechanism. We abbreviate \(Q\)-advantage as \(QA\) and call the above mechanism \textbf{TD\textit{QA}} to illustrate it combines two metrics as the measurement of human guidance. The \(QA\) term is removed for non-demonstration tuples when calculating the above equation, thus, the priority levels of non-demonstration data are aligned with those in the conventional PER.

In this manner, the experience in the buffer \(\mathcal{B}\) subjects to a distribution \(\mathcal{I}^\prime\), and the probability mass function of the experience distribution \(\mathbf{p}_{\mathcal{I}^\prime }\sim\mathcal{I}^\prime\) can be expressed as:
\begin{equation}\label{eq14}
    \mathbf{p}_{\mathcal{I}^\prime } (i)=\frac{\mathbf{p}_i^\alpha}{\sum_k\mathbf{p}_k^\alpha}.
\end{equation}

We inherent the optimization trick of the conventional PER by using a sum-tree structure to store transition data, and the updating and sampling can be conducted with a complexity of \(O(\log N)\). 

The priority mechanism introduces the bias to the estimation of the expectation of the value function since it changes the experience distribution in the buffer. Biased value network \(Q\) could have little impact on the RL asymptotic performance, yet it may affect the stability and robustness of the mature policy in some situations. As an optional operation, we can anneal the bias by introducing the importance-sampling weight to the loss function of the value network. The importance-sampling weight of a transition i is calculated as:
\begin{equation}\label{eq15}
    w_{IS} (i)=\left[\mathbf{p}_{\mathcal{I}^\prime}(i) \right]^{-\beta} .
\end{equation}
where \(\beta\in[0,1]\) is a coefficient: the fully non-uniform sampling occurs if \(\beta=1\), and fully uniform sampling occurs if \(\beta=0\). \(\beta\) will gradually decrease to zero along with the training process.

The importance-sampling weight can be added to the loss function of the value network, expressed as:
\begin{equation}\label{eq16}
\begin{split}
    &\mathcal{L}^Q(\theta)=\\
&\underset{\zeta_{i\sim \mathcal{I}^\prime }}{\mathbb{E}} w_{IS} (i)\left(r_i+\gamma Q(\mathbf{s}_{i+1},\pi(\cdot\vert \mathbf{s}_{i+1});\theta)
-Q\left(\mathbf{s}_{i},\mathbf{a}_{i};\theta\right)\right)].
\end{split}
\end{equation}
Through the proposed PER, we prioritize human guidance over RL experiences. Moreover, high-quality demonstrations are prioritized to more extents, and the utilization efficiency of human demonstrations can be enhanced.

\subsection{Human-intervention-based Reward Shaping}\label{section4.4}
In this section, we introduce the human-intervention-based reward shaping technique. Naturally, there is no need for humans to provide guidance if the being-trained RL agent is executing a good policy. Therefore, to minimize the human workload, we assume human participants would intervene in the training process only when RL’s behaviors are unfavorable. In this context, the intervention event can be seen as a negative signal and the corresponding state should be avoided by RL. This negative feedback can be realized by reward shaping, which will be detailed in this section. 

\begin{figure}[tbh]
\begin{center}
\noindent
  \includegraphics[width=\linewidth]{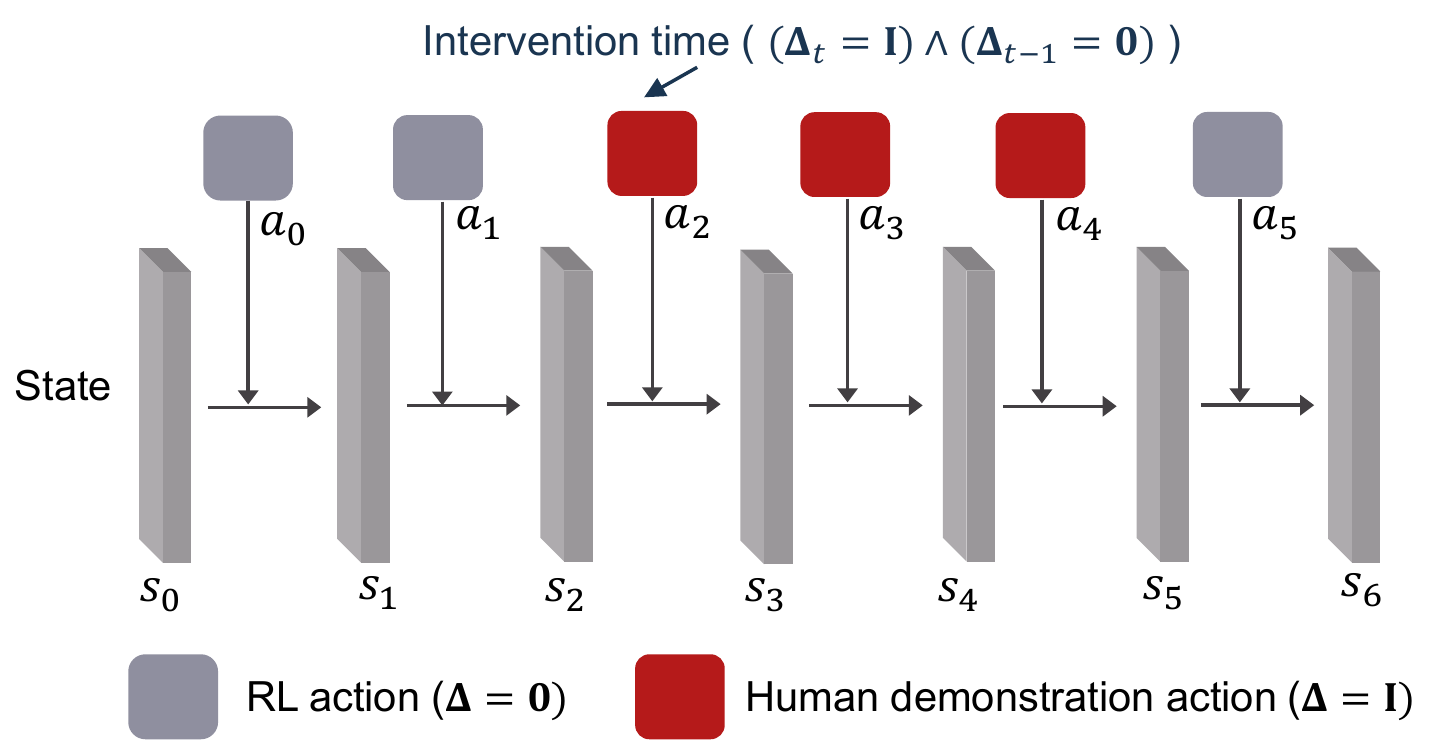}
  \end{center}
    \caption[Illustration of the intervention time step.]{Illustration of the intervention time step. In a time-sequential MDP, the first time step which is controlled by human demonstration is taken as the intervention time. }\label{Figure2}
\end{figure}

We first identify the intervention event. Recall Eq. \ref{eq9} defines a mask \(\mathbf{\Delta}_t\), which is a time-sequential variable recording if the action \(\mathbf{a}_t\) is conducted by human demonstration. Hence, the intervention time, i.e., the start time of a period of human demonstrations, can be represented by \((\mathbf{\Delta}_t=\mathbf{I})\wedge(\mathbf{\Delta}_{t-1}=\mathbf{0})\) in a time-sequential training process of RL, as illustrated in Fig. \ref{Figure2}. It is noted that only the intervention time is to be punished by the reward shaping, since the states after humans intervention will be substituted by human demonstrations and cannot be seen as unfavorable. For instance, in Fig. \ref{Figure2}, \(\mathbf{s}_2\) is penalized while \(\mathbf{s}_3\) and \(\mathbf{s}_4\) are not.

Then, we can shape the vanilla reward function with an additional penalized function:
\begin{equation}\label{eq17}
    r_t^{\text{shape}}=r_t+r_{\text{pen}}[(\mathbf{\Delta}_t=\mathbf{I}^{\text{dim}(\mathcal{A} )}\land(\mathbf{\Delta}_{t-1}=\mathbf{0}^{\text{dim}(\mathcal{A})} )],
\end{equation}
where \(r_t^{shape}\) is the reward after shaping, \(r_{pen}\) is a scalar that weighs the intervention penalty. 

The theoretical performance of this reward shaping scheme is analyzed in Appendix \ref{appendix}.

\subsection{Prioritized human-in-the-loop RL algorithm}\label{section4.5}
In this section, we integrate all the above components and propose a holistic RL algorithm considering human guidance. It is noted that although the human guidance-based actor-critic framework in Section \ref{section4.2} and reward shaping in Section \ref{section4.4} are components of the algorithm, they are not the major novelty of this report. To highlight our core idea of the prioritized human-demonstration replay mechanism of Section \ref{section4.3}, we name the proposed algorithm as \textbf{P}rioritized \textbf{H}uman-\textbf{I}n-the-\textbf{L}oop (\textbf{PHIL}) RL. 

Specifically, we obtain the holistic human-in-the-loop RL configuration through equipping the human-guidance-based actor-critic framework with prioritized human-demonstration replay and intervention-based reward shaping mechanisms. We instantiate the PHIL algorithm based on one of the state-of-the-art off-policy RLs, i.e., twin delayed deep deterministic policy gradient (TD3) \cite{RN26}. We also remind the above components are adaptive to various off-policy actor-critic RL algorithms.

In TD3, the target networks, namely, the target critic \(Q^\prime\) with parameter \(\theta^\prime\) and target actor \(\pi^\prime\) with parameter \(\phi^\prime\) are utilized to stabilize the algorithm update. And the actor’s output becomes a deterministic value instead of a sample from the probability distribution. 

Considering the role of human participants in the RL interaction process, the eventual action in the time step \(t\) can be expressed as:
\begin{subequations}\label{eq18}
\begin{equation}
    \mathbf{a}_t=(\mathbf{I}^{\dim(\mathcal{A})} -\mathbf{\Delta}_t )\cdot\mathbf{a}_t^{RL}+\mathbf{\Delta}_t\cdot\mathbf{a}_t^H,
\end{equation}
\begin{equation}
    \mathbf{a}_t^{RL}=\pi(\cdot\vert\mathbf{s}_t)+\text{clip}\left(\epsilon,-c,c\right), \epsilon\sim\mathcal{N}\left(\mathbf{0},\mathbf{\Sigma} \right),
\end{equation}
\end{subequations}
where \(\epsilon\) is a noise coefficient vector dependent on the training proceed, \(c\) is the bounding of the exploratory action, \(\mathbf{\Sigma}\) is the covariance matrix of the Gaussian distribution \(\mathcal{N}\).

A transition tuple is obtained through the above interaction step and stored into the proposed human-demonstration experience buffer as:
\begin{equation}\label{eq19}
    \mathcal{B}\gets\zeta_t=(\mathbf{s}_t,\mathbf{a}_t,r_t,\mathbf{s}_{t+1},\mathbf{\Delta}_t).
\end{equation}
Stored experience tuples will be retrieved for the training of the value and policy networks. An arbitrary transition tuple \(\zeta\) with index \(i\) would be retrieved by the probability \(p\), which is calculated by:
\begin{equation}\label{eq20}
    p(i)=\frac{\mathbf{p}_i^\alpha}{\sum_k\mathbf{p}_k^\alpha },
\end{equation}
wherein the priority level \(\mathbf{p}\) is:
\begin{subequations}\label{eq21}
\begin{equation}
    \mathbf{p}_t=\vert\delta_t^{TD} \vert + \varepsilon+(\mathbf{\Delta}_t=\mathbf{I}^{\dim(\mathcal{A})} )\cdot QA,
\end{equation}
\begin{equation}
    QA=\exp\left[Q^\prime \left(\mathbf{s}_t,\mathbf{a}_t ;\theta^\prime \right)-Q^\prime\left(\mathbf{s}_t,\pi(\cdot\vert\mathbf{s}_t;\phi);\theta^\prime \right)\right],
\end{equation}
\end{subequations}

It is noticeable that \(Q\)-advantage is calculated by the target critic network \(Q^\prime\) to avoid unstable updates.

Supposing a tuple with size \(N\) contains \(N_1\) amount of non demonstration tuples and \(N_2=N-N_1\) human demonstration ones, the loss function of the critic can be expressed as:
\begin{align}\label{eq22}
\begin{split}
        &\mathcal{L}^{Q_k} (\theta)=
        \\&\frac{1}{N_1} \sum_i^{N_1}\Vert r_i+\gamma Q_l^\prime \left(\mathbf{s}_{i+1},\pi^\prime (\cdot\vert \mathbf{s}_{i+1} ) \right)-Q_k \left(\mathbf{s}_i,\mathbf{a}_i^{RL}\right)\Vert_2^2 
        \\&+\frac{1}{N_2} \sum_j^{N_2}\Vert r_j+\gamma Q_l^\prime \left(\mathbf{s}_{j+1},\pi^\prime \left(\cdot\vert \mathbf{s}_{j+1} \right) \right)-Q_k \left(\mathbf{s}_j,\mathbf{a}_j^H\right)\Vert_2^2
\end{split}
\end{align}
where \(k=1,2\) represents the index of two \(Q\) networks. Note the double \(Q\) network trick, which utilizes the smaller \(Q\) value of two networks (\(l=\min\{1,2\}\)), is introduced here to eliminate the value overestimation effect.

The loss function of the actor is calculated as:
\begin{equation}\label{eq23}
\begin{split}
    \mathcal{L}^\pi(\phi)&=
\frac{1}{N_1}\sum_i^{N_1}\left[-Q_1 \left(\mathbf{s}_i,\pi\left(\cdot\vert\mathbf{s}_i;\phi\right);\theta\right)\right] \\
&+\frac{1}{N_2}  \sum_j^{N_2}\left[\omega\cdot\Vert \mathbf{a}_j^H-\pi\left(\cdot\vert \mathbf{s}_j;\phi\right)\Vert_2^2\right].
\end{split}
\end{equation}

It is noticeable that the training of the policy network can be delayed stabilizing the algorithm, that is, the actor would be updated once given the critic updating \(d\) times.

Lumping all factors, the complete version of the proposed algorithm is provided in Algorithm \ref{alg1}.

\SetKwInOut{Input}{Input}
\begin{algorithm}[tb]
\caption{PHIL-TD3 }\label{alg1}
\Input{maximum episode number $E$, episode duration $T$, batch size $N$, policy network $\pi(\cdot;\phi)$, value networks $Q_1(\cdot;\theta_1),Q_2(\cdot;\theta_2)$, target networks, empty buffer \(\mathcal{B}\), learning rate $lr^Q, lr^{\pi}$, priority coefficient $\alpha$, policy update factor $d$, soft update factor $\tau$. }

 \For{episode=1 to \(E\)}{
  Observe the initial state \(\mathbf{s}_1\)\;
  \For{t=1 to \(T\)}{
    \uIf{human intervene}{
    Adopt human action \(\mathbf{a}_t=\mathbf{a}_t^{\mathcal{H}}\), set \(\mathbf{\Delta}_t=\mathbf{I}\)\;
    }
    \Else{
    Select RL action \(\mathbf{a}_t=\mathbf{a}_t^{RL}=\pi(\cdot\vert\mathbf{s}_t;\phi)+\epsilon\), set \(\mathbf{\Delta}_t=\mathbf{0}\) \;
    }
    Observe reward \(r_t\) and new state \(\mathbf{s}_{t+1}\) \;
    Shape reward \(r_t=r_t+r_{pen}\cdot[(\mathbf{\Delta}_t=\mathbf{I})\land(\mathbf{\Delta}_{t-1}=\mathbf{0})]\)\;
    Store tuple \(\left(\mathbf{s}_t,\mathbf{a}_t,r_t,\mathbf{s}_{t+1}^{\prime},\mathbf{\Delta}_t\right)\) in \(\mathcal{B}\) with priority \(\mathbf{p}_t= \max_{i<t}\mathbf{p}_{i}\) \;
    Sample \(N\) tuples from \(\mathcal{B}\) with probability  \(p(i)=\mathbf{p}_i^\alpha/(\sum_k\mathbf{p}_k^\alpha)\)\;
    Update priority by Eq. \ref{eq13} \;
    Update value networks by \(\theta_{k=1,2}\gets\theta_{k=1,2}-lr^{Q_{k=1,2} }\cdot\nabla_\theta \mathcal{L}^Q (\theta) \)\; 
    \If{\(t\) mod \(d\)}{
    Update policy network by \(\phi\gets\phi-lr^{\pi}\cdot\nabla_{\phi} \mathcal{L}^{\pi} (\phi)\)\;
    Update target networks \;
    }
    }
 }
\end{algorithm}

\section{Human Policy Model}\label{section5}

In this section, a human policy model is established in conjunction with PHIL-RL. The model can relieve human workload in the human-in-the-loop RL process by imitating the behavior policy of actual human participants. 

We train a regression model to imitate human policy simultaneously with RL, and this policy model can substitute humans when necessary. Consider human behaviors in the RL training process: the human participant is required to intervene in the control process when he/she believes the agent poorly behaves. Human interventions are usually imposed to the loop in an intermittent way and demonstrations are incrementally supplemented into the training set (buffer). Thus, we train the human policy model leveraging an online- and incremental-based imitation learning algorithm, i.e., the Data Aggregation (DAgger) \cite{RN27}, which is free from offline large-scale collection of the demonstration data.

It is noted that the human policy model does not aim to accurately mimic expert-level humans. In practice, the common situation is humans who cooperate with RL are non-proficient, and humans' performance can fluctuate with mental and physical status. Thus, we do not require the model to achieve expert-level performance. In essence, the human policy model is to provide roughly correct demonstrations for the RL agent.

Denoting the human policy model with \(\mathcal{H}\), the objective is to find a policy \(\pi^\mathcal{H}\) minimizing its difference \(\mathbf{d}\) with the human policy \(\pi^H\):
\begin{equation}\label{24}
    \pi^\mathcal{H}=\arg\min_{\pi}\left[\mathbb{E}_{\mathbf{s}_i} \left[\mathbf{d}(\mathbf{s}_i,\pi^H )\right]\right].
\end{equation}

We initialize model \(\mathcal{H}\) by replicating an untrained RL policy network. After the first human-intervention event, model \(\mathcal{H}\) is established as:
\begin{equation}\label{eq25}
    \pi_0^\mathcal{H}\ (\varphi)\gets\pi(\phi).
\end{equation}
In subsequent episodes, we retrieve human demonstrations to conduct incremental learning with the loss function:
\begin{equation}\label{eq26}
    \mathcal{L}^\mathcal{H} (\varphi)=\mathbb{E}_{(\mathbf{s}_i,\mathbf{a}_i^H)} \left[\Vert \mathbf{a}_{i}^H-\pi^H (\cdot\vert \mathbf{s}_{i};\phi)\Vert_2^2 \right],
\end{equation}
and update the model with the gradient method as:
\begin{equation}\label{eq27}
    \pi_{e+1}^\mathcal{H}\gets\pi_e^\mathcal{H}-lr^{\pi^\mathcal{H}} \cdot\nabla_\varphi\ \mathcal{L}^\mathcal{H}\ (\varphi),
\end{equation}
where \(e\) is the episode number of the RL process. 

Through the above update, model \(\mathcal{H}\) would gradually be competent to accurately mimic human policy, and accordingly, substitute human participants to assist RL. It is noticeable that if using this human policy model to cooperate with PHIL, the activation conditions of model \(\mathcal{H}\) shall be manually defined varying to specific environments.

\section{Problem Formulation}\label{section6}

The proposed PHIL-TD3, like most RLs, can be universally adapted to any continuous-action decision and control tasks. Here we choose the end-to-end autonomous driving problem as the object, evaluating our algorithm in two challenging driving scenarios. Note that the RL-based autonomous driving problem can be solved by numerous reasonable settings, while the problem formulation in this section is to provide a fair environment for algorithm evaluation and comparison. 

In this section, two challenging autonomous driving scenarios are introduced to evaluate the control and optimization performance of the proposed algorithm, then the standard optimization setting is established.

\subsection{Autonomous Driving Scenarios}\label{section6.1}
RL is better suited to the challenging driving tasks compared to rule-based or model optimization-based approaches due to its high representational and generalization capabilities. We choose two scenarios, shown in Fig. \ref{Figure3}, to evaluate the RL performance. These scenarios are challenging to conventional autonomous driving strategies due to complex combinatorial relationships. 

\subsubsection*{Unprotected left-turn}
This scenario is illustrated in Figs. \ref{Figure3}(a-b). The ego vehicle, i.e., the controlled vehicle, in the side road is trying to make a left turn and merge into the main road. No traffic signals guide the vehicles in the intersection. We assume the lateral path of the ego vehicle is planned by other techniques, while the longitudinal control is assigned to the RL agent. Surrounding vehicles are initialized with varying random velocities ranging from [4, 6] m/s and controlled by the intelligent driver model (IDM) \cite{RN28} to execute lane-keeping behaviors. All surrounding drivers are set with aggressive characteristics, meaning that they would not yield to the ego vehicle. The control interval for all vehicles is set as 0.1 seconds. 

\begin{figure}[b]
\begin{center}
\noindent
  \includegraphics[width=\linewidth]{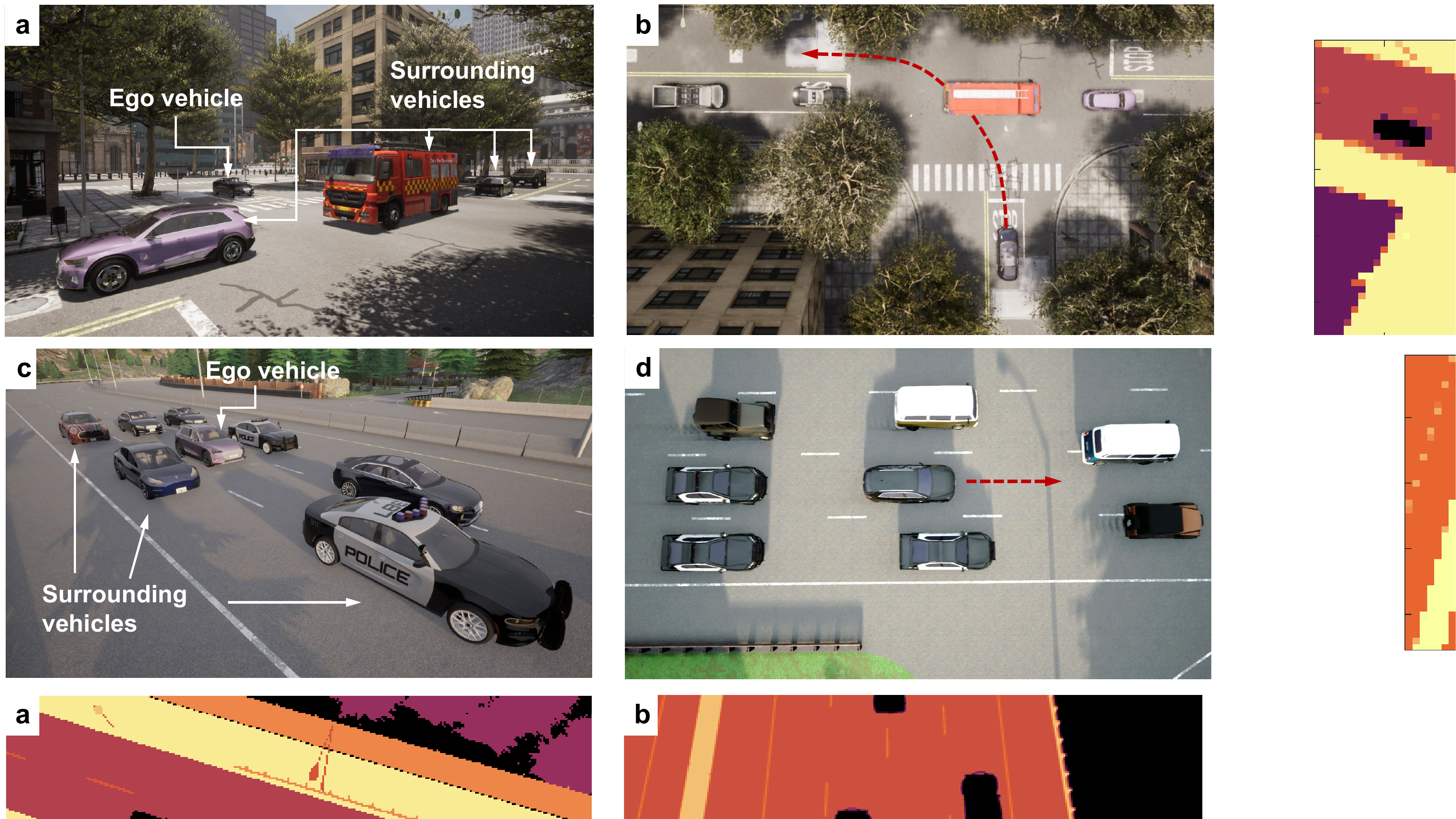}
  \end{center}
    \caption[Task environment configuration.]{Task environment configuration. a, the devised unprotected left-turn scenario in T-intersection, established in CARLA. b, the bird-view of the left-turn scenario, where the dotted line indicates a left-turn trajectory. c, the devised congestion scenario in the highway, established in CARLA. d, the bird-view of the congestion scenario, where the dotted line shows a car-following trajectory.}
    \label{Figure3}
\end{figure}

\subsubsection*{Highway congestion}
This scenario is illustrated in Figs. \ref{Figure3}(c-d). The ego vehicle is stuck in severe congestion and tightly surrounded by other vehicles; thus, it is trying to shrink the gap with its leading vehicles and conduct the car-following task with the target velocity. We assume the longitudinal control is completed by IDM with a target velocity of 6 m/s, while the lateral control is assigned to the RL agent. Surrounding vehicles are initialized with the velocity ranging from [4, 6] m/s and controlled by IDM to execute car-following behaviors. The control interval for all vehicles is set as 0.1 seconds. The crowded surrounding vehicles cover the lane markings and no specific one leading vehicle in the ego lane, which can lead the conventional lateral-planning approaches to be invalid in such a scenario.

\begin{figure}[tb]
\begin{center}
\noindent
  \includegraphics[width=\linewidth]{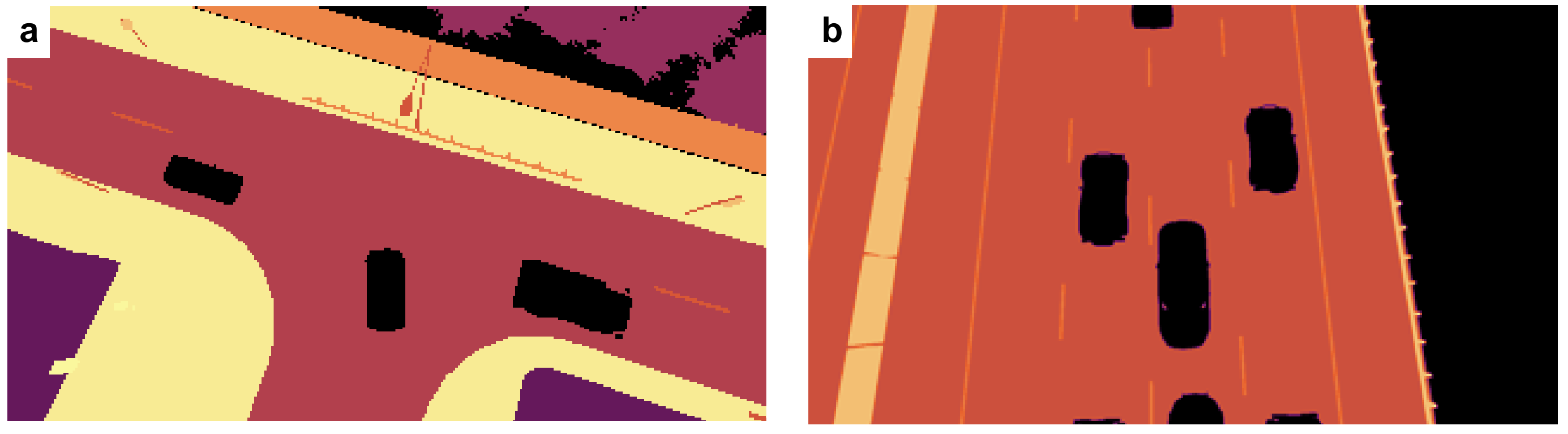}
  \end{center}
    \caption[State space illustration: bird's-eye-view semantic graph.]{State space illustration: bird's-eye-view semantic graph. a, the left-turn scenario, b, the congestion scenario. }\label{Figure4}
\end{figure}

\subsection{RL-based problem definition}\label{section6.2}

\subsubsection*{State}: The bird-view semantic graphs are taken as the state information for the RL agent, shown in Fig. \ref{Figure4}. Two consecutive frame images are used to constitute one state variable to enable temporal perception. We scale the camera-captured image to a smaller size to relieve the computational burden. The state variable can be expressed as:
\begin{equation}\label{eq28}
    \mathbf{s}_t=\{\mathbf{p}_{t-1}, \mathbf{p}_t \vert \mathbf{p}\in [0,1]\},
\end{equation}
where \(\mathbf{p}\in\mathbb{R}^{45\times80}\) is a pixel matrix of which the elements are normalized. 

\subsubsection*{Action}: The action variable can be either lateral or longitudinal commands adaptive to different requirements. For the lateral control task in the congestion scenario, we choose the angle of the steering wheel as the action, expressed as:
\begin{equation}\label{eq29}
    \mathbf{a}_t=\left[\delta_t \vert \delta\in \left[-5\lambda\pi,5\lambda\pi \right]\right],
\end{equation}
where \(\delta\in\mathbb{R}^1\) is the continuous steering command, of which the negative value indicates a left-turn command and the positive value corresponds to a right-turn command, and \(\lambda\) is the scaling factor that limits the steering range. 

For the longitudinal control in the left-turn scenario, we choose the accelerating/braking pedal aperture, expressed as:
\begin{equation}\label{eq30}
    \mathbf{a}_t=\left[\eta_t\ \vert \eta\in \left[-1,1\right]\right],
\end{equation}
where \(\eta\in\mathbb{R}^1\) is the continuous pedal aperture, of which the negative value indicates a braking command and the positive value corresponds to an accelerating command.

\subsubsection*{Reward}: The goal of an autonomous vehicle is to rapidly complete traffic scenarios through safe and smooth driving behaviors. RL-based driving strategy achieves this by an appropriate reward function definition. The reward schemes of the two tasks in Fig.\ref{Figure3} can be respectively defined as:
\begin{equation}\label{eq31}
\begin{split}
    R^{\text{left-turn}} \left(\cdot\vert \mathbf{s}_t,\mathbf{a}_t \right)=
&r_{\text{goal}}\cdot\mathbf{1}\left(\mathbf{s}_t\in\mathcal{S}_{\text{goal}} \right)\\
+&r_{\text{fail}}\cdot\mathbf{1}\left(\mathbf{s}_t\in\mathcal{S}_{\text{fail}} \right)+r_{\text{speed}} \left(\mathbf{s}_t\right) ,
\end{split}
\end{equation}
\begin{equation}\label{eq32}
\begin{split}
    R^{\text{congestion}} \left(\cdot\vert \mathbf{s}_t,\mathbf{a}_t \right)=
&r_{\text{goal}}\cdot\mathbf{1}\left(\mathbf{s}_t\in\mathcal{S}_{\text{goal}} \right)\\
+&r_{\text{fail}}\cdot\mathbf{1}\left(\mathbf{s}_t\in\mathcal{S}_{\text{fail}} \right)+r_{\text{steer}} \left(\mathbf{s}_t\right),
\end{split}
\end{equation}
where \(r_{\text{goal}} = 10\) and \(\mathcal{S}_{\text{goal}}\) is the set of goal states where the ego vehicle successfully completes the scenario; \(r_{\text{fail}} = -10\) and \(\mathcal{S}_{\text{fail}}\) is the set of failure states where the collision occurs; while \(r_{\text{speed}} = -\Vert v_{\text{ego}} - v_{\text{target}}\Vert\) is the reward that encourages the target speed, i.e., \(5m/s\) set in this section; \(r_{\text{steer}} = \Vert \delta_t - \delta_{t-1}\Vert\) is the reward that discourages frequent steering behaviors. It is noticeable that both \(r_{\text{speed}}\) and \(r_{\text{steer}}\) can implicitly play a role in promoting smooth driving. Additionally, we set the penalty term \(r_{\text{pen}}\) in Eq. \ref{eq17} the same as \(r_{\text{fail}}\) and incremented it to the above reward when human intervention occurs.

\subsubsection*{Function approximator}: The function approximators of the value and policy functions are concrete by deep convolutional networks, as shown in Fig. \ref{Figure5}.

\subsubsection*{Auxiliary functions}: We define some auxiliary control functions independent of the RL action to achieve a complete control suit. When RL manipulates the steering wheel, the longitudinal control is achieved by an IDM. When RL manipulates the pedal aperture, the lateral motion target is to track the planned waypoints through a proportional-integral (PI) controller.

\begin{figure}[tb]
\begin{center}
\noindent
  \includegraphics[width=\linewidth]{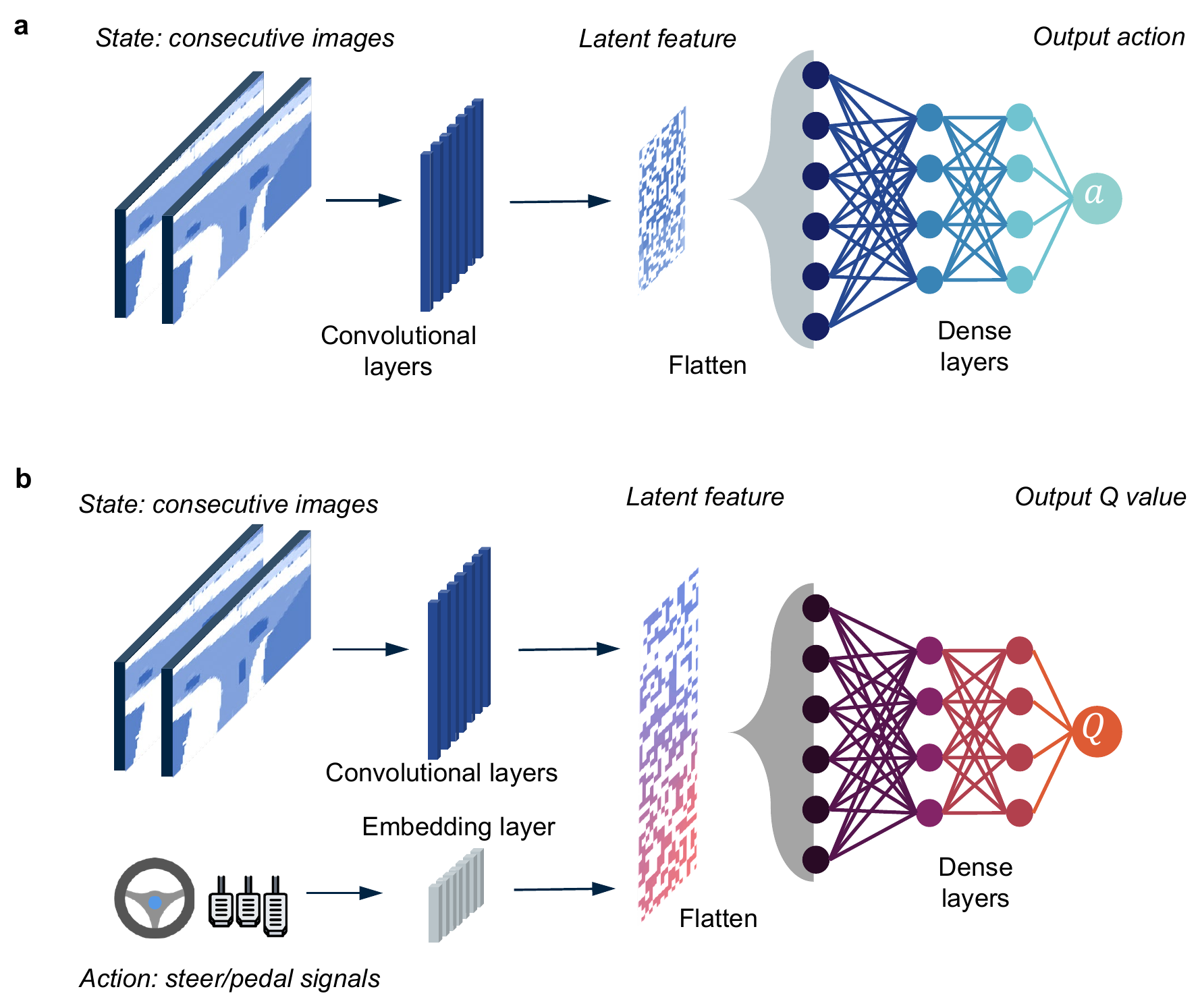}
  \end{center}
    \caption[Neural network approximator illustration.]{Neural network approximator illustration. a, the policy function architecture achieved by the neural network, where the target value network owning the same structure is omitted for brevity. b, the value network architecture achieved by the neural network, where the target value network owning the same structure is omitted for brevity.  }\label{Figure5}
\end{figure}

\section{Experimental Validation}\label{section7}
\subsection{Baseline Algorithms}\label{section7.1}

We employ state-of-the-art in the domain of human-involved RL algorithms as baselines and compare their performance against the proposed algorithm. 

\subsubsection*{IA-TD3}
This baseline is derived from Intervention Aided Reinforcement Learning (IARL) , which is a representative combination of a continuous-action RL algorithm and human demonstration. The RL’s policy network is modified to adapt to human demonstrated actions by introducing the behavior cloning objective. Once human intervention happens, the human demonstration will substitute the RL’s exploratory action, and a penalty signal will impose on the reward value. In this study, we devise a modified IARL by replacing the on-policy base algorithm with TD3, which essentially augmented the algorithm by improving the sample efficiency. We also implement the prioritized experience replay (PER) in this baseline for a fair comparison. 

\subsubsection*{HI-TD3}
This baseline is derived from Human Intervention Reinforcement Learning (HIRL) , which is a combination of a discrete-action RL algorithm and human demonstration. Once intervention happens, the human demonstration will substitute the RL’s exploratory action, and a penalty signal will take on the reward signal. In this study, we devise a modified HIRL by replacing the discrete-action base algorithm with TD3, which augmented the algorithm by improving the representation and control precision. We also implement the PER in this baseline for a fair comparison.

\subsubsection*{RD2-TD3}
This baseline is derived from Recurrent Replay Distributed Demonstration-based DQN (R2D3), which is a representative combination of PER mechanism and human demonstration. In this study, we devise a modified algorithm by replacing DQN with TD3. The original R2D3 utilizes the recurrent neural network to augment performance, which is not the concerned technique in the context of this report, thus, we remove the recurrent network structure and only focus on its replay distributed character regarding human demonstrations. Thus, we devise a Replay Distributed Demonstration-based (RD2) TD3 algorithm, which distributes human demonstration and RL exploratory experience into two experience buffers respectively and retrieves experiences by PER. The probability of utilizing human guidance instead of RL exploratory experience is aligned with the ratio of human guidance amount and total data amount.

Furthermore, we use the vanilla PER+TD3 that is shielded from human guidance as an ablated baseline.

\subsection{Experimental Setting}\label{section7.2}

Multiple experiments are to evaluate the comprehensive performance of PHIL-TD3 against baselines. First, the training efforts of involved algorithms are comparatively evaluated in the two autonomous driving scenarios. Then the well-trained autonomous driving strategies are tested regarding control performance with several metrics. Last, a series of experiments involving both training and testing stages are conducted to analyze the mechanism of PHIL-TD3.

\begin{figure}[hb]
\begin{center}
\noindent
  \includegraphics[width=\linewidth]{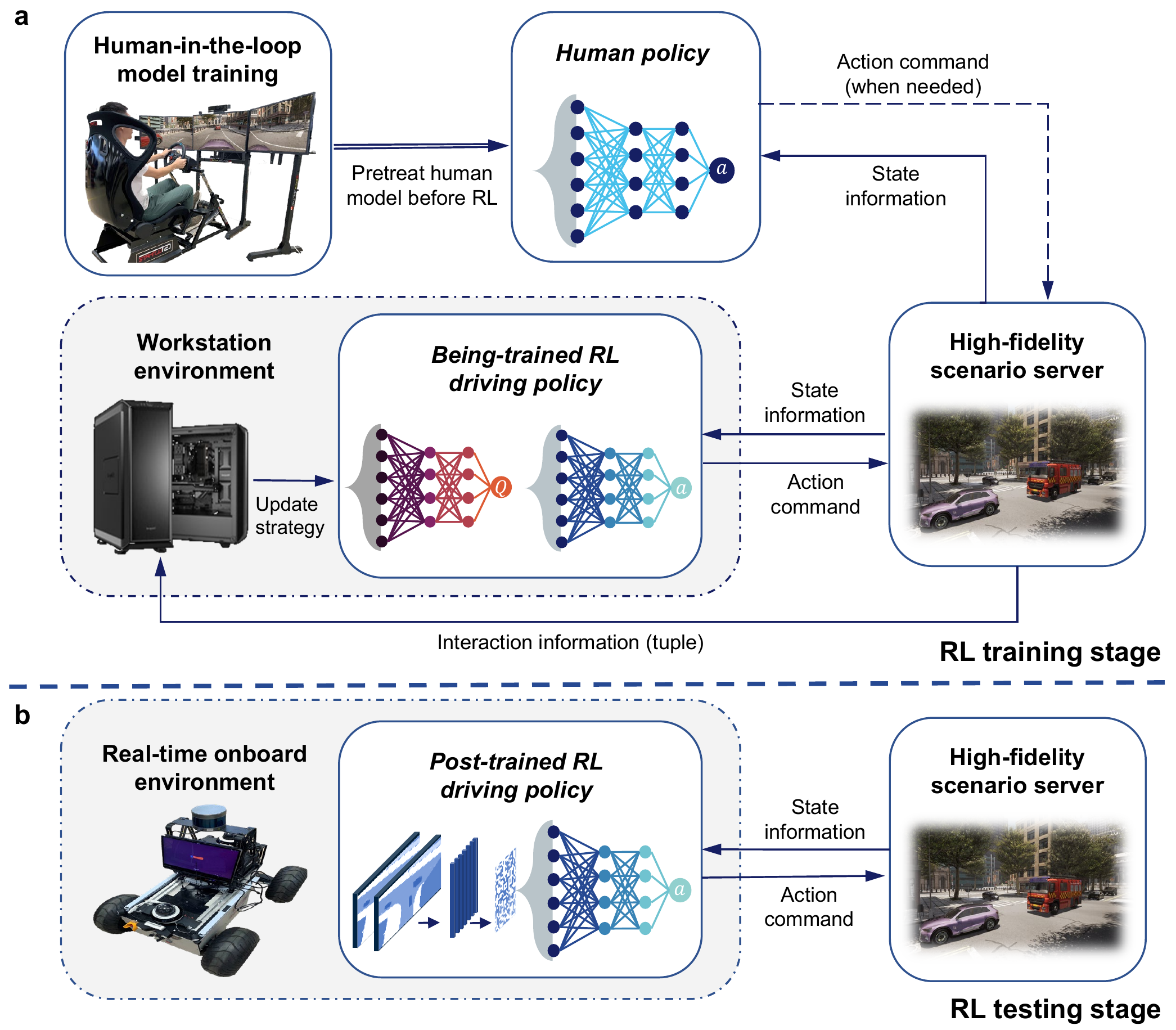}
  \end{center}
    \caption[Experimental workflow]{Experimental workflow. a, the experimental workflow in the RL training stage. The dotted line represents the human policy model that is not always sending commands. b, the experimental workflow in the RL testing stage.  }\label{Figure6}
\end{figure}

The training hardware comprises a driving simulator and a high-performance workstation. The driving simulator is utilized to collect human data to train the human policy model complying with Section IV, and the workstation is dedicated to processing RL training. A high-fidelity autonomous driving simulation platform, CARLA \cite{RN29}, is employed to implement driving scenarios and generate RL-environment interaction information. The schematic diagram of the RL training stage is illustrated in Fig. \ref{Figure6}(a).

The testing hardware is a robotic vehicle. The post-trained RL policy is implemented on the computation platform of the vehicle, which can communicate with the CARLA server through the wireless network. The on-board RL policy receives state information from CARLA and sends its control command back to remotely complete autonomous driving tasks. The robotic vehicle aims to test whether the RL policy is well-worked under the current onboard calculation and communication situations. The schematic diagram of the RL testing stage is demonstrated in Fig. \ref{Figure6}(b).

The detailed configuration of the above experimental platform is provided in table \ref{tableA1}. The algorithms are concreted based on neural networks, of which the architecture is illustrated in Appendix \ref{appendix}. And the hyperparameters of the algorithms are also given in Appendix \ref{appendix}.

\subsection{Evaluation of RL Training Performance}\label{section7.3}

In this section, we explore whether human guidance can indeed improve the RL training, and further, which algorithm can achieve the best learning performance given the same human guidance. Additionally, we also investigate the effects of human guidance in dealing with RL tasks of different difficulties. 

To eliminate the deviation brought by participant randomness and obtain repeatable results, we use the identical human model (see Section \ref{section5}) to mimic human guidance behaviors in RL training processes. We fixate the sequence of random seeds and make the triggering conditions of human interventions invariant in all training attempts, which achieves a fair comparison across different algorithms. Two metrics are employed: the average reward of the training episode (excluding intervention-based shaping term), and the surviving distance of the ego vehicle in the training episode before a goal state or failure state in Eq. \ref{eq31} occurs. A higher value of both metrics indicates a better learning performance.

\begin{figure}[tb]
\begin{center}
\noindent
  \includegraphics[width=\linewidth]{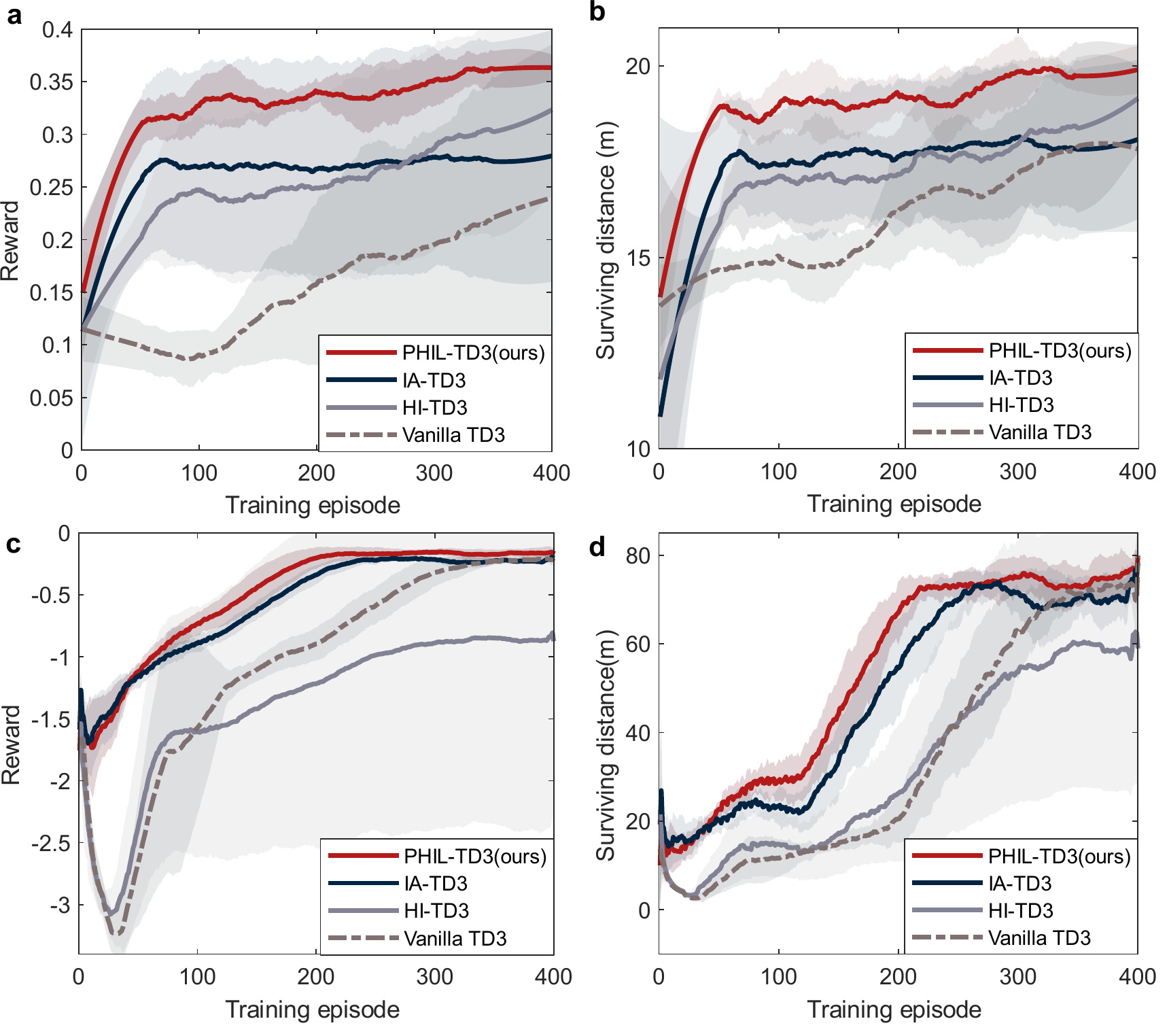}
  \end{center}
    \caption[Learning efforts of different RL algorithms.]{Learning efforts of different RL algorithms. a-b, curves of training rewards and surviving distances in the left-turn scenario, respectively. c-d, curves of training rewards and surviving distances in the congestion scenario, respectively. }\label{Figure7}
\end{figure}

\begin{figure}[hb]
\begin{center}
\noindent
  \includegraphics[width=\linewidth]{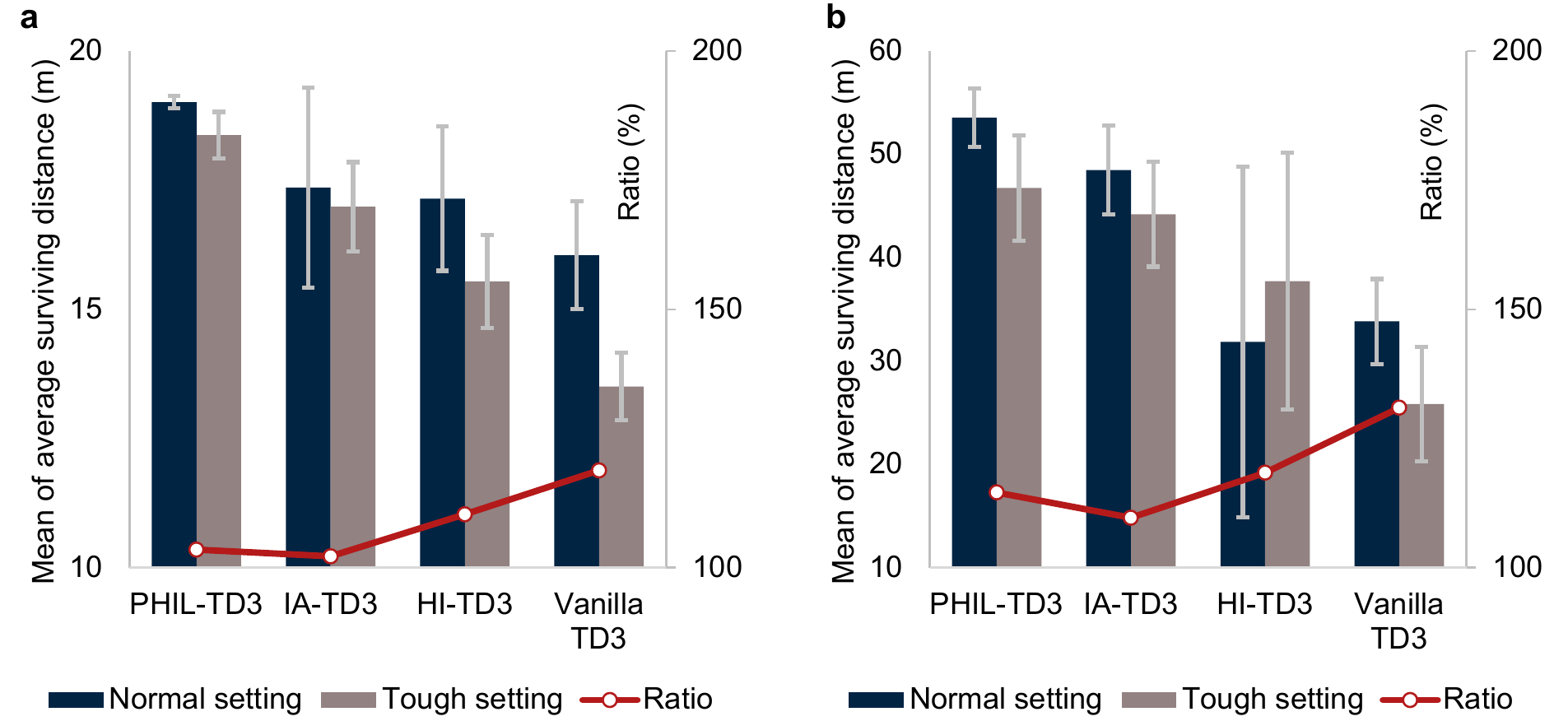}
  \end{center}
    \caption[Learning efforts of different RL algorithms in different task difficulties.]{Learning efforts of different RL algorithms in different task difficulties. a-b, curves of training rewards and surviving distances in the left-turn scenario, respectively. c-d, curves of training rewards and surviving distances in the congestion scenario, respectively. }\label{Figure8}
\end{figure}

Fig. \ref{Figure7} visualizes the learning performance through curves, represented with a solid line of the mean value and an error band of the standard deviation. We run each algorithm five times in the unprotected left-turn scenario and demonstrate their learning processes in Figs. \ref{Figure7}(a-b). The vanilla TD3 is struggling to improve its policy, while the other three algorithms achieve higher rewards and survive distances in a much shorter time, which indicates the effectiveness of human guidance. Among the human-involved algorithms, HI-TD3 performs the slowest learning process suggested by either reward or surviving distance, and IA-TD3 exhibits a faster convergence but with limited asymptotic performance. In opposite, PHIL rapidly seizes the opportunity of human guidance and learns the best asymptotic policy. It is noticeable that PHIL-TD3 achieves the best asymptotic average reward of the baselines in less than 50 episodes, improving the learning efficiency by over 700\%. We also run the congestion scenario five times for each algorithm and plot the learning curves in Figs. \ref{Figure7}(c-d). The comparable PHIL and IA-TD3 perform better than the other two baselines when considering the reward. While the metric of surviving distance further confirms this advantage and profitably differentiates the algorithm abilities. Specifically, PHIL wins the highest eventual score. IA-TD3 and HI-TD3 manifest comparable levels of asymptotic performance while IA-TD3 has an advantage in learning efficiency. In this scenario, PHIL-TD3 achieves the best asymptotic average surviving distance of the baselines in 220 episodes, improving the learning efficiency by over 120\%. Overall, the results in this training session highlight the significant superiority of the proposed algorithm in learning performance.

We further explore the learning performance of RLs with different task difficulties, which gives rise to Fig. \ref{Figure8}. The normal setting complies with the problem definition in Section \ref{section6}, which is adopted throughout the report, while the tough setting changes consecutive-frame input of Eq. \ref{eq28} into a single frame input, impairing the temporal perception ability of RL agents. At the high level, the statistical results of the normal setting are aligned with the trends of Figs. \ref{Figure7}(a-d). And it is indicated that the tough setting does not change the performance ranking of algorithms despite the degradations in different degrees. At the detail level, the performance difference between the normal and the tough settings, i.e., the ratios in Fig. \ref{Figure8}, can manifest more algorithmic characteristics. Specifically, PHIL-TD3 and IA-TD3, which own the behavior-cloning objective, are less affected by the incomplete problem definition of the tough setting, whereas HI-TD3, and vanilla TD3, which less or not rely on human guidance, are significantly degraded in the same condition. Despite the single-frame state in the autonomous driving task is not fairly reasonable, the findings through this comparison are useful. Since numerous complex real-world tasks are intractable to be well-defined or are only partially observable, the strong integration of human guidance into RL, e.g., behavior-cloning, can play a more remarkable role than pure RL algorithms.

Then, we investigate the contributions of different components in improving the performance of the proposed PHIL algorithm. The results are provided in Appendix 1. Three components, the behavior cloning objective of Eq. \ref{eq12} of Section \ref{section4.2}, the proposed prioritized experience replay mechanism of Section \ref{section4.3}, and the intervention-based reward shaping mechanism of Section \ref{section4.4}, are validated to be effective, respectively. The results show that the proposed prioritized human-demonstration replay mechanism plays a crucial role in improving the ultimate performance. 

Last, we evaluate the computational efficiency. The CPU clock time of different algorithms is compared in table \ref{tableA8}. It is shown that the training time consumption of the proposed algorithm is similar to that of IA-TD3. This is because the proposed priority calculation scheme consumes very few computational resources. In all, the proposed PHIL-TD3 greatly improves the training efficiency and performance without requiring significantly higher computational resources.

\begin{table}[tb]
\caption{Comparison of time efficiency of different RL algorithms}
\begin{center}
\begin{tabular}{ cc } 
 \hline
 Algorithm & Time consumption (s) per 10000 steps \\ 
 \hline
 PHIL-TD3 & 360.70\\
 IA-TD3 & 348.71 \\ 
 HI-TD3 & 328.93 \\
 Vanilla-TD3 & 329.39\\
 \hline
\label{tableA8}
\end{tabular}
\end{center}
\end{table}

\subsection{Evaluation of Testing Performance of Driving Strategies}\label{section7.4}

In this section, the post-trained driving strategies are tested in terms of autonomous driving performance, adaptiveness, and robustness, which can further evaluate the practicality of the above algorithms.

\begin{figure}[tb]
\begin{center}
\noindent
  \includegraphics[width=\linewidth]{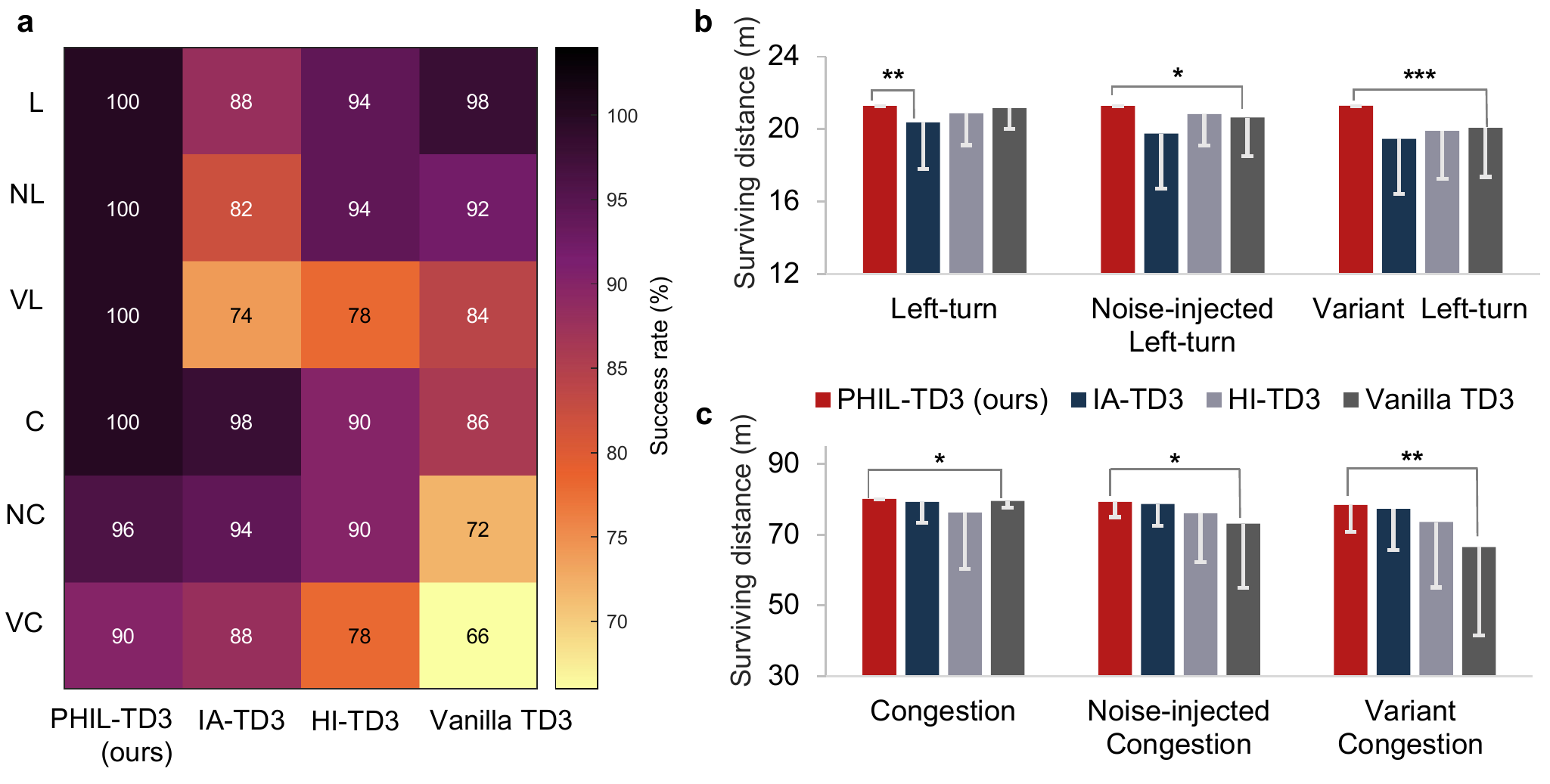}
  \end{center}
    \caption[High-level driving performance of different RL strategies under six autonomous driving scenarios.]{High-level driving performance of different RL strategies under six autonomous driving scenarios. The two noise-injected scenarios and two variant scenarios are different with the two training scenarios, which can examine the robustness and adaptiveness, respectively. “C” and “L” refer to congestion scene and left-turn scene, respectively, while “N” and “V” denote noise-injected and variant scene, respectively. a, the heatmap of success rate. b, the barplot of surviving distance in the left-turn scenarios. The theoretical maximum surviving distance of the scenario is 21 meters. The error bar describes the standard deviation. c, the barplot of surviving distance in the congestion scenarios. The theoretical maximum surviving distance of the scenario is 80 meters. The error bar describes the standard deviation. The paired t-test is adopted for the statistical test. }\label{Figure9}
\end{figure}

The zero-mean Gaussian noises, of which the standard deviation is 5\% of the whole control domain, are injected to output commands of the driving strategies to test the robustness. More types and amounts of surrounding vehicles are added to construct variant scenarios to test the adaptiveness. We conduct 50 runs with the same sequence of random seeds for each post-trained strategy in each scenario. The success rate, which is defined as the number of completed runs divided by the total attempts in the same scenario, is taken as the metric for evaluating the safety performance in Fig. \ref{Figure9}(a). Our PHIL-TD3 achieves the highest success rate in all scenarios, showing its superior task-completeness abilities. The vanilla TD3, albeit with its unstable training performance, performs competitively like IA-TD3 and HI-TD3 in the testing stage. Considering the two trained scenes (rows 1, 4) and noise-injected scenes (rows 2, 5), three baseline strategies behave acceptably, nevertheless, the scenario variants (rows 3, 6) significantly degrade their safety. Our PHIL, instead, maintains the highest ability regardless of varying testing conditions, manifesting itself with good robustness and adaptiveness. In Figs. \ref{Figure9}(b-c), PHIL-TD3 once again shows its superiority in safety by the highest average surviving distance, and importantly, its performance stability is confirmed due to the lowest variance. 

Fig. \ref{Figure10} can further evaluate the detailed performance of driving strategies. Time consumption of the episode is the secondary target of RL optimization in the left-turn task, which is implied in the reward function of Eq. \ref{eq31}; thus, the related boxplot is illustrated in Fig. \ref{Figure10}(a) to access this objective. It is found that the proposed strategy enjoys minimal time consumption, which is significantly different from other candidates. In congestion tasks, smoothness is the secondary target of the reward function of Eq. \ref{eq32}; thus, we choose the lateral acceleration as the smoothness measure and provide the associated boxplot in Fig. \ref{Figure10}(b). The comparable human-involved strategies show their superior smoothness to vanilla TD3 in the training and noise-injected scenes, while the variant congestion scenario profitably validates the advantage of PHIL-TD3. 

\begin{figure}[tb]
\begin{center}
\noindent
  \includegraphics[width=\linewidth]{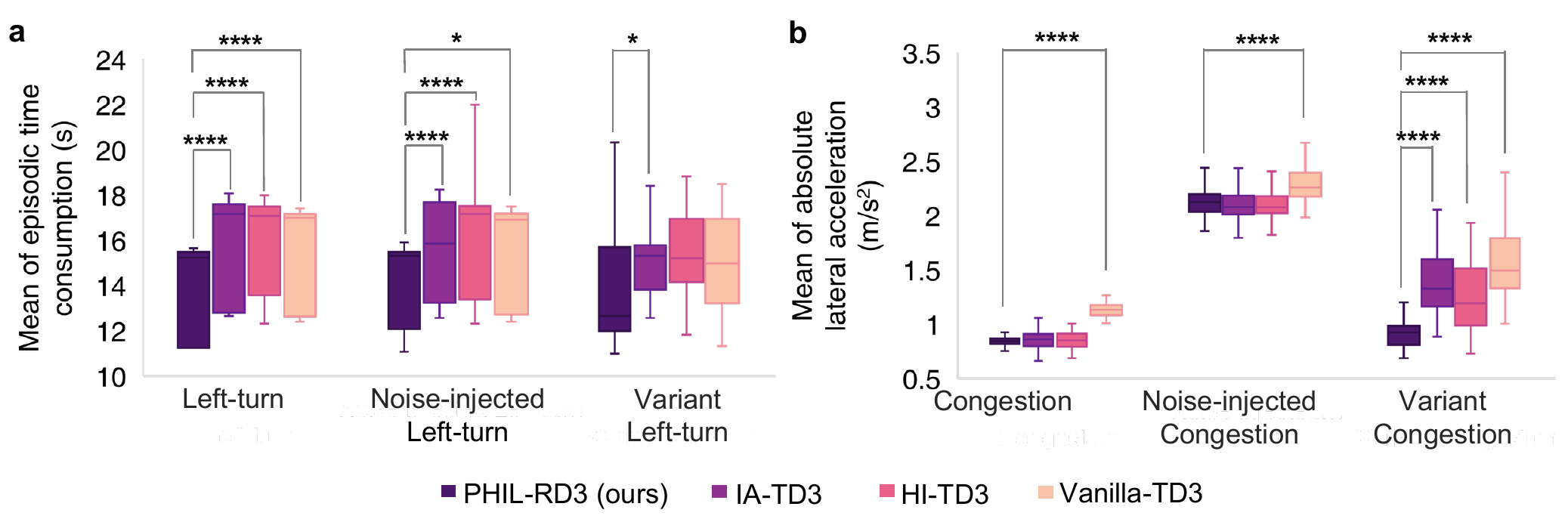}
  \end{center}
    \caption[Low-level driving performance of different RL strategies under six autonomous driving scenarios.]{Low-level driving performance of different RL strategies under six autonomous driving scenarios.a, the boxplot of time consumption of the episode without failure in the left-turn scenarios. b, the boxplot of average lateral acceleration of the episode in the congestion scenarios. The paired t-test is adopted for the statistical test. }\label{Figure10}
\end{figure}

Additionally, we compare the performance of three human guidance-based RL algorithms to the human guidance itself. Specifically, the surviving distance of these human-involved RLs are compared with the human policy model, and the results are provided in Tables \ref{table1}-\ref{table2}, where the results (mean and standard deviation) are calculated by 50 evaluation seeds. The results suggest the superiority of the proposed PHIL-TD3 over the human policy model.

Overall, our PHIL-TD3 perpetuates its predominance of training performance and takes the top spot in the testing stage.

\begin{table*}[tb]
\caption{Comparison of the surviving distance of human-related driving strategies in three left-turn scenarios.}
\begin{center}
\begin{tabular}{ cccc } 
 \hline
 Surviving distance, meter, \(\uparrow\) & Left-turn & Noise-injected Left-turn & Variant Left-turn \\ 
 \hline
 PHIL-TD3 (ours) & \textbf{21.28\(\pm\)0.02} & \textbf{21.27\(\pm\)0.02} & \textbf{21.29\(\pm\)0.02} \\ 
 IA-TD3 & 20.37\(\pm\)2.58 & 19.75\(\pm\)3.05 & 19.46\(\pm\)3.03\\ 
 HI-TD3 & 20.87\(\pm\)1.76 & 20.63\(\pm\)1.74 & 19.90\(\pm\)2.65 \\
 Human policy model & 20.70\(\pm\)1.82 & 20.88\(\pm\)1.32 & 20.90\(\pm\)1.21 \\
 \hline
\label{table1}
\end{tabular}
\end{center}
\end{table*}

\begin{table*}[tb]
\caption{Comparison of the surviving distance of human-related driving strategies in three congestion scenarios.}
\begin{center}
\begin{tabular}{ cccc } 
 \hline
 Surviving distance, meter, \(\uparrow\) & Congestion & Noise-injected Congestion & Variant Congestion \\ 
 \hline
 PHIL-TD3 (ours) & \textbf{80.15\(\pm\)0.08} & \textbf{79.26\(\pm\)4.30} & \textbf{77.29\(\pm\)11.55}\\ 
 IA-TD3 & 79.26\(\pm\)5.90 & 78.64\(\pm\)6.11 & 78.39\(\pm\)7.66 \\ 
 HI-TD3 & 76.27\(\pm\)16.00 & 76.02\(\pm\)13.72 & 73.57\(\pm\)18.49 \\
 Human policy model & 80.11\(\pm\)0.07 & 77.66\(\pm\)12.27 & 75.15\(\pm\)15.20\\
 \hline
\label{table2}
\end{tabular}
\end{center}
\end{table*}

\subsection{Discussion on Prioritized Human Experience Mechanism}\label{section7.5}

In this section, we explore the effect of PHIL-RL from three aspects: the performance improvement by the \(TDQA\) mechanism, the merit of the single-buffer experience replay structure, and the algorithmic robustness to bad demonstrations.

\(TDQA\), as the crucial innovation of PHIL-TD3, can improve learning performance in the context of human guidance-based RLs, as suggested in Sections \ref{section7.3} and \ref{section7.4}. More specifically, it establishes a novel priority indicator to deal with various human guidance. Thus, we first evaluate \(TDQA\) by comparing different priority schemes. “\(Q\)-adv” represents the scheme in which the priority of human guidance is calculated based only on \(Q\)-advantage. “\(TD\)”, i.e., temporal difference, the scheme is inherited from the original PER method, but the \(TD\) weights of human demonstrations in it are doubled to highlight the human guidance in the replay buffer. 

Five learning attempts are conducted with the same sequence of random seeds for each candidate, and the corresponding learning curves are in Fig. \ref{Figure11}. We find scheme comparison in two training scenarios shows similar trends when observing results in Figs. \ref{Figure11}(a-b) and Figs. \ref{Figure11}(c-d). The pure \(TD\) scheme learns faster than pure \(Q\)-advantage in both scenarios, yet its asymptotic scores (both reward and surviving distance) are significantly lower than those of the \(Q\)-advantage scheme. To be more specific, we evaluate different weigh of “\(TD\)” and “\(Q\)-adv” and provide the learning performance in Fig. \ref{Figure12}. Under the same \(TD\), \(Q\)-advantage is weighted with three importance levels. In particular, the equal weighting scheme, i.e., \(w=1e0\), is the adopted default scheme in the report, whereas the other two variants are for comparison. It is shown that a larger \(TD\) (\(w=1e-1\)) makes faster convergence but can lead to unfavorable asymptotic performance, while a larger \(Q\)-advantage (\(w=1e2\)) can achieve the same-level performance as the default setting, despite sometimes slower learning process. The above results, reveal the same performance trends as Fig. \ref{Figure11}. That is, \(TD\) error accelerates the convergence speed and \(Q\)-advantage contributes to improving convergence performance.

\begin{figure}[tb]
\begin{center}
\noindent
  \includegraphics[width=\linewidth]{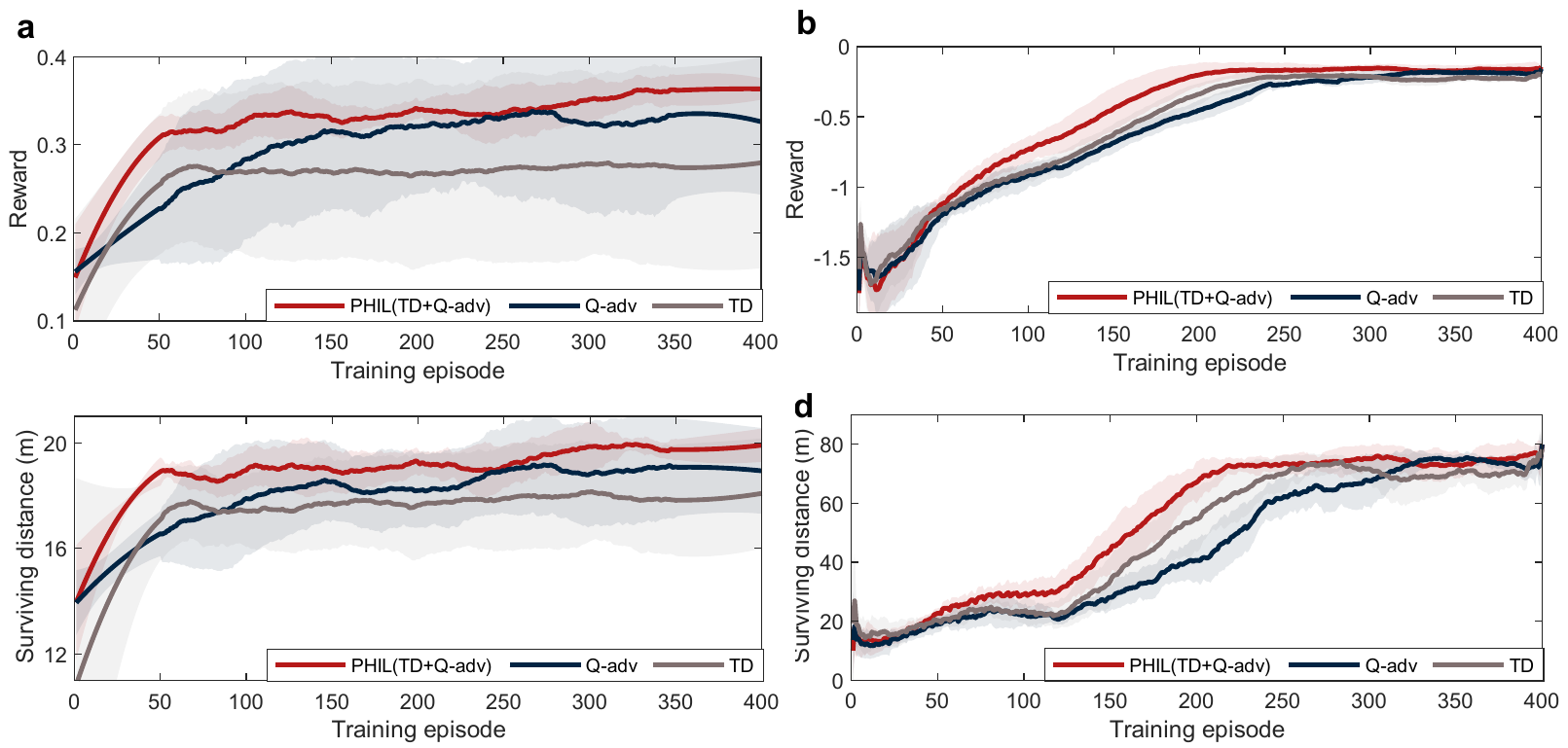}
  \end{center}
    \caption[Learning efforts of experience replay mechanisms with different priority indicators.]{Learning efforts of experience replay mechanisms with different priority indicators.. a-b, training rewards in left-turn and congestion scenario, respectively. c-d, surviving distances in left-turn and congestion scenario, respectively.  }\label{Figure11}
\end{figure}

\begin{figure}[tb]
\begin{center}
\noindent
  \includegraphics[width=\linewidth]{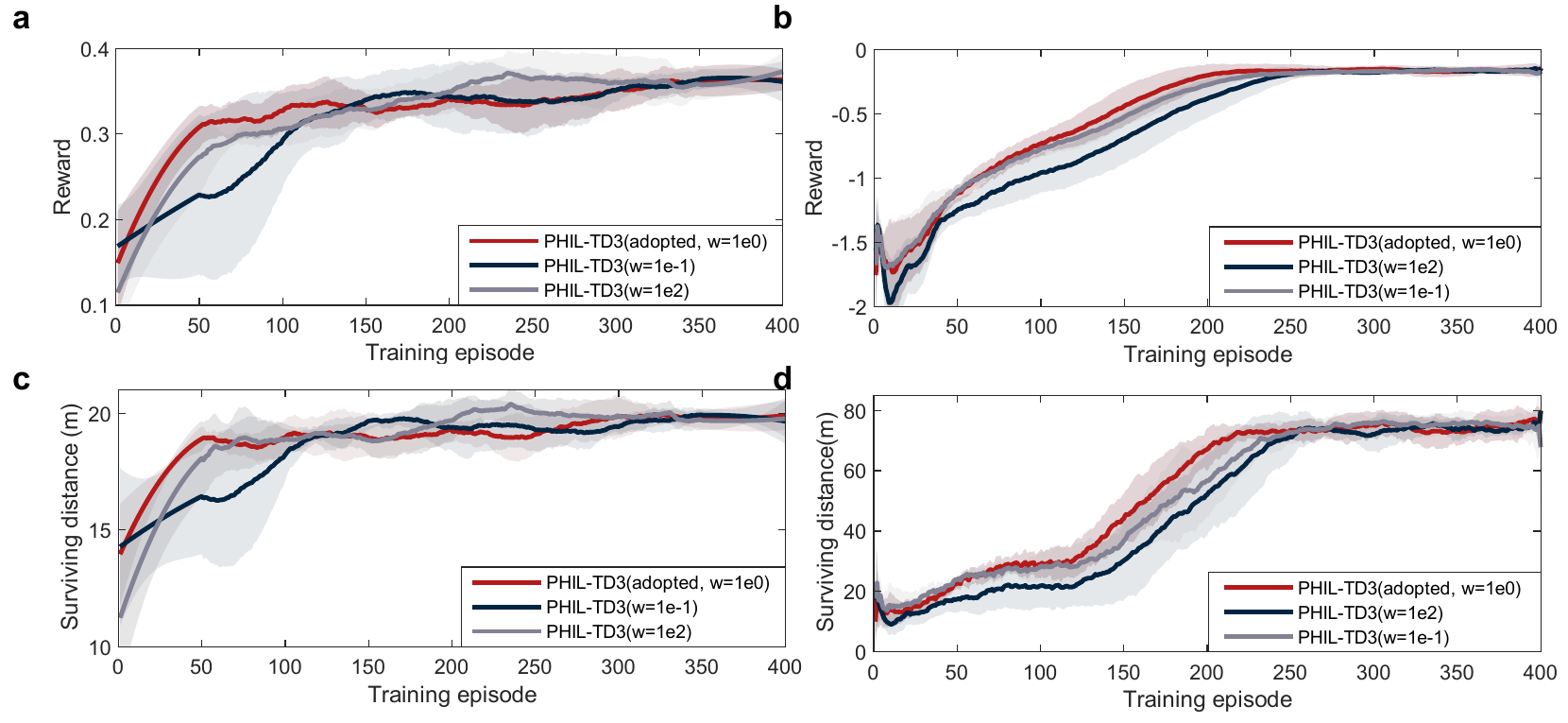}
  \end{center}
    \caption[Learning efforts of experience replay mechanisms with different weighting schemes.]{Learning efforts of experience replay mechanisms with different weighting schemes. a-b, training rewards in left-turn and congestion scenario, respectively. c-d, surviving distances in left-turn and congestion scenario, respectively.  }\label{Figure12}
\end{figure}

Essentially, these two schemes score human guidance based on different indicators, and a better indicator can provide RL with more high-quality guidance to improve learning efficiency.  Thus, we find \(TD\) indicator, as proved in conventional PER, is indeed beneficial to rapidly improve performance, nonetheless, the \(Q\)-advantage indicator is superior to the \(TD\) indicator in the later stage of the training process. The delayed superiority of \(Q\)-advantage complies with intuition since unlike the direct indicator as \(TD\), the evaluation ability of the \(Q\) network, i.e., the source of \(Q\)-advantage, also needs to be trained. The proposed PHIL, which smartly combines both indicators, achieves the most favorable performance in the two scenarios, showing the effectiveness of the \(TDQA\) mechanism.

\begin{figure}[tb]
\begin{center}
\noindent
  \includegraphics[width=\linewidth]{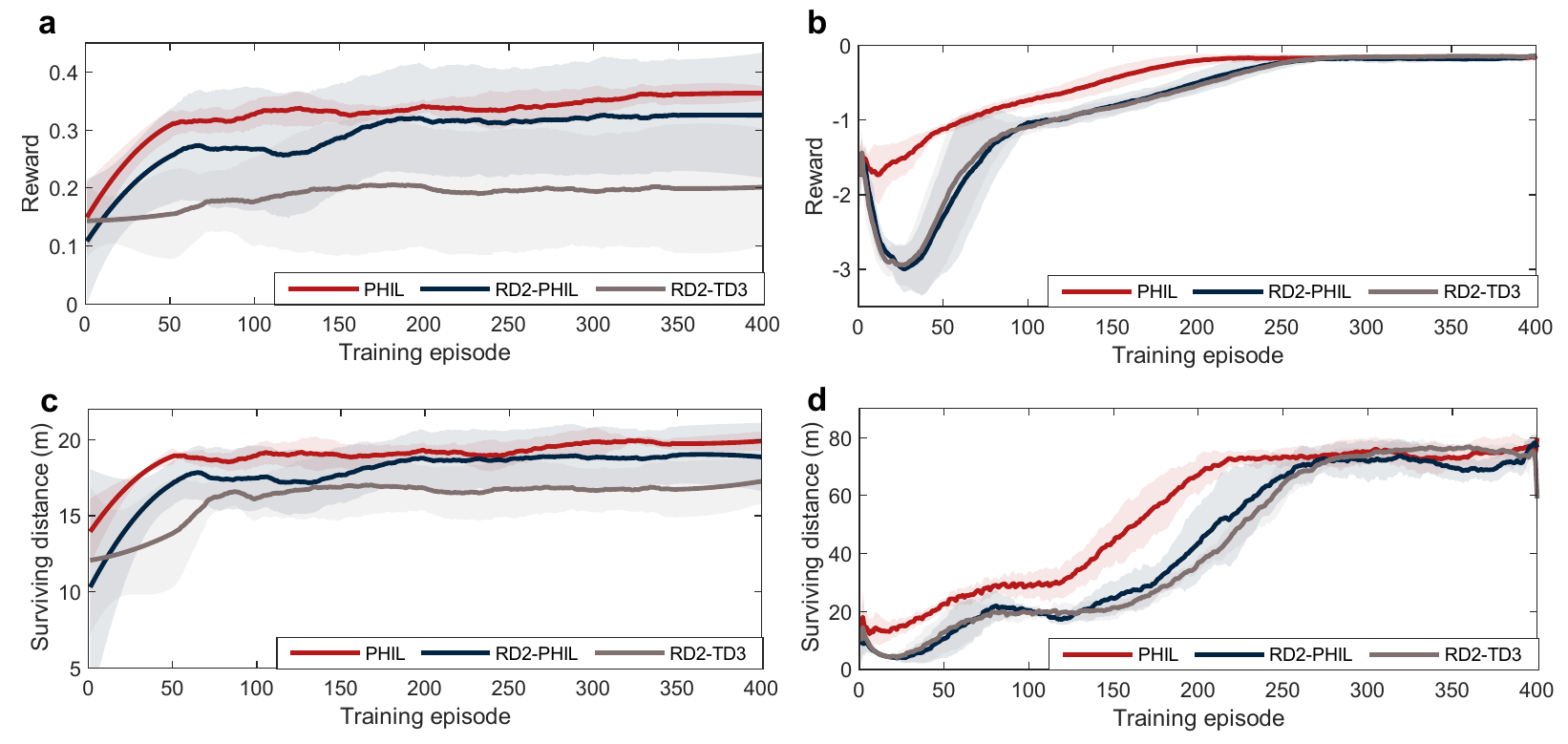}
  \end{center}
    \caption[Learning efforts of experience replay mechanism with different replay schemes]{Learning efforts of experience replay mechanism with different buffer structure. a-b, the training rewards of algorithms with different experience replay structures in the left-turn scenario and congestion, respectively. c-d, the training surviving distances of algorithms with different experience replay structures in the left-turn and congestion scenario, respectively.  }\label{Figure13}
\end{figure}

PHIL puts the human guidance and exploratory experience of RL into the same experience replay buffer. This structure differs from the double distributed scheme which is represented by R2D3. To evaluate the performance of these two schemes under the devised autonomous driving tasks, RD2-TD3 is developed which utilizes \(TD\) as the indicator to respectively retrieve data from two buffers. Additionally, the \(TDQA\) priority mechanism is ported to the RD2-TD3 setting forming the other variant, RD2-PHIL. Five learning attempts with the same sequence of random seeds are conducted by RD2-TD3 and RD2-PHIL. Through learning curves in Figs. \ref{Figure13}(a-d), it is found that the double distributed buffer scheme, i.e., RD2-PHIL, fails to achieve the same level of learning efficiency as the proposed PHIL. A possible reason behind this is that human guidance can only be utilized in a chunk way under the double-buffer setting, whereas the single buffer scheme of PHIL is more flexible and friendly to small-scale human guidance data. The conventional RD2-TD3 is least favorable, which is within expectation due to the lack of the \(TDQA\) mechanism. To sum up, the results in Fig. \ref{Figure13} support the single-buffer structure utilized in the PHIL-TD3, and profitably suggest the effectiveness of the proposed TDQA mechanism.

A general situation occurs that human guidance is not perfect, and thus an unqualified human participant can sometimes conduct actions that are harmful to the task. We test if the unfavorable guidance of the unqualified human would impair the learning process, that is, evaluating the robustness to harmful guidance. It is noticeable that the robustness discussed here is distinguished from that in Section \ref{section7.4}: we discuss how the algorithms are affected by poor guidance instead of the anti-noise ability of post-trained driving strategies. The human intervention condition of the training stage keeps the same as foregoing experiments, while one-third of the demonstrations from the human model are replaced with random actions to simulate non-proficient human behaviors.

\begin{figure}[tb]
\begin{center}
\noindent
  \includegraphics[width=\linewidth]{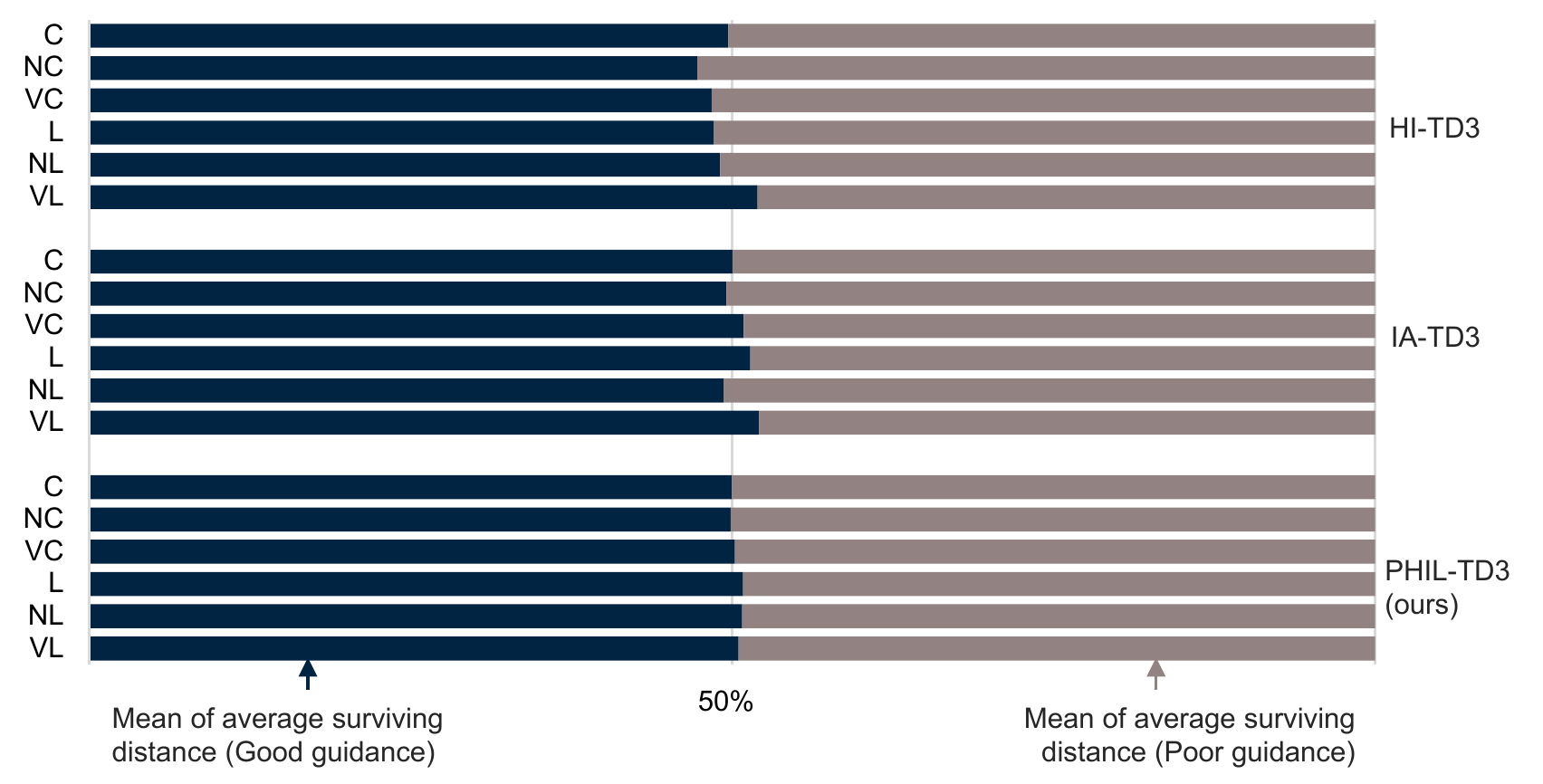}
  \end{center}
    \caption[Robustness comparison of different human-guided RL strategies with respect to the quality of human guidance data: surviving distance.]{Stacked barplot of the surviving distance of different human-guided RL strategies under good/poor guidance in all scenarios. }\label{Figure14}
\end{figure}

Post-trained driving strategies under poor guidance are tested to conduct 50 runs in each scenario and are compared with those under the good guidance of Fig. \ref{Figure9}. The stacked barplots in Fig. \ref{Figure14} provide the adversarial testing performance of three human-guidance-based RL algorithms under good and poor guidance. We take the average surviving distance as the metric and the less performance deterioration by poor guidance suggests better robustness. Our PHIL-TD3 exhibits good performance since a nearly 50:50 situation occurs in all six scenarios. IA-TD3 falls behind with a 2.1\% degradation on average in poor guidance context, while HI-TD3 is even improved by an average of 3.6\% extent given poor guidance. Intuitively, poor guidance would remarkably degrade PHIL and IA-TD3 since they utilize the behavior-cloning objective to learn from human guidance, while HI-TD3, which only substitutes partial RL explorations with human guidance, can be less affected. The not-degraded HI-TD3 and most-degraded IA-TD3 support the above idea. Our PHIL defeating IA-TD3 is attributed to the \(TDQA\) mechanism: \(Q\)-advantage well access the quality of human demonstrations and feed more high-quality demonstrations to the RL agent; accordingly, the agent learns greater from good guidance than negative guidance. The secondary optimization target of RL, i.e., driving smoothness, is evaluated in Fig. 15 by acceleration distribution. The proposed PHIL-TD3 wins all scenarios by the most favorable smoothness which further confirms the abovementioned superiority. 

Overall, the \(TDQA\) mechanism, as the core innovation of the PHIL-RL algorithm, contributes to the preponderant learning performance through its unique discriminatory power on the quality of human guidance. It also improves the robustness to poor guidance, which can relieve the requirement on the qualification of human guidance. Additionally, the single buffer setting is more favorable than the double distributed buffer scheme under autonomous driving tasks of this report.

\begin{figure}[tb]
\begin{center}
\noindent
  \includegraphics[width=\linewidth]{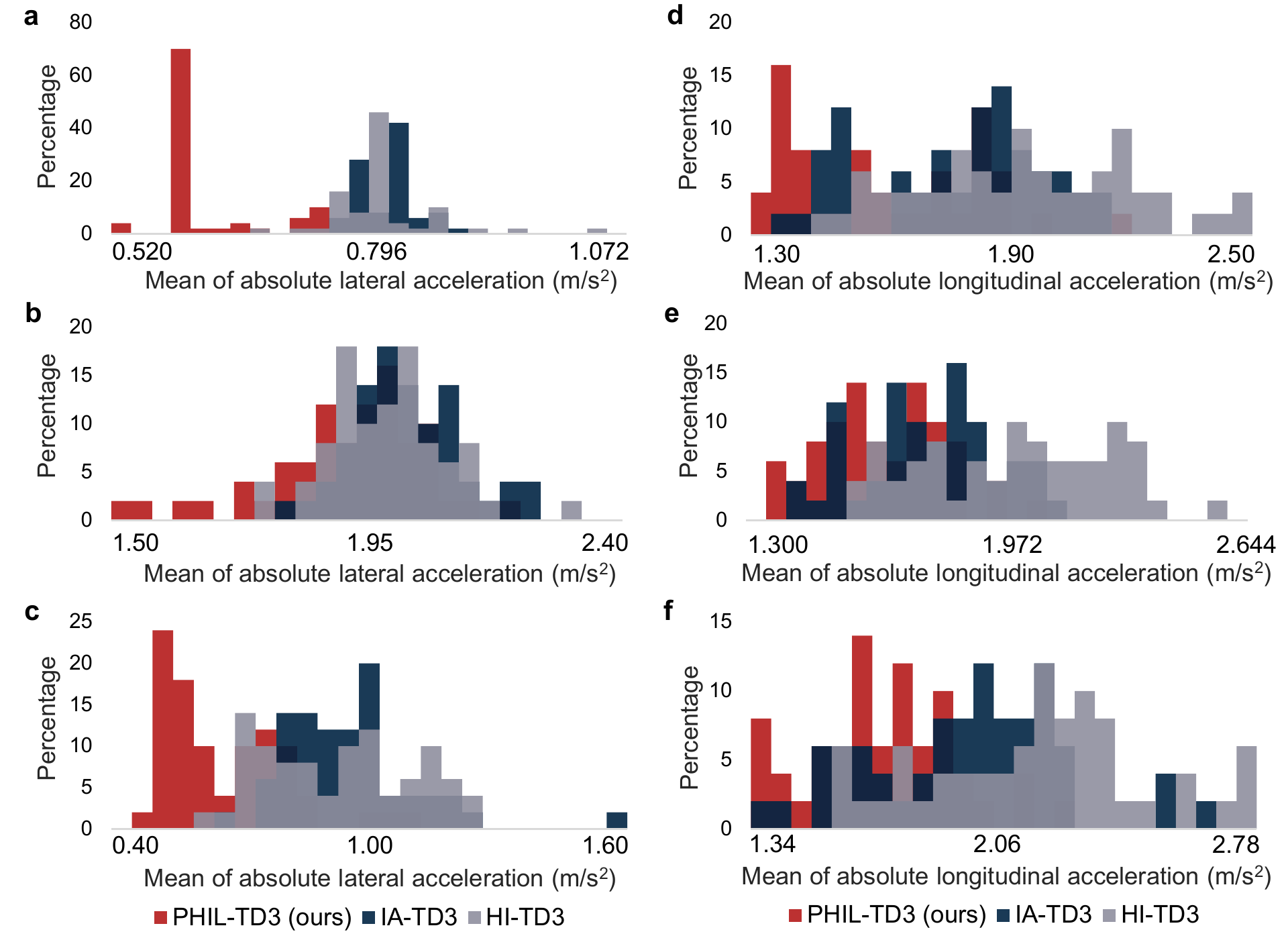}
  \end{center}
    \caption[Robustness comparison of different human-guided RL strategies with respect to the quality of human guidance data: acceleration.]{Acceleration distributions of different human-guided RL strategies. a-c, the frequency distribution plot of the average absolute value of the longitudinal acceleration in the left-turn scene, noise-injected left-turn scene, and variant left-turn scene, respectively. The smaller acceleration indicates a better driving smoothness. d-f, t  he frequency distribution plot of the average absolute value of the lateral acceleration in the congestion scene, noise-injected congestion scene, and variant congestion scene, respectively. The smaller acceleration indicates a better driving smoothness. }\label{Figure15}
\end{figure}

\section{Conclusions}\label{section8}
In this paper, we establish a human-guidance-based reinforcement learning framework and propose a novel experience utilization mechanism of human guidance. Based on that, we put forward an algorithm, PHIL-TD3, aiming at improving algorithmic abilities in the context of human-in-the-loop RL. We also introduce a human behavior modeling mechanism to relieve the human workload. PHIL-TD3 is employed to solve two challenging autonomous driving tasks, and its performance is comparatively evaluated against state-of-the-art human-guidance-based RLs as well as the non-guidance baseline. Three main points are obtained through experimental results:

1) The proposed PHIL-TD3 can improve the learning efficiency by over 700\% and 120\% under the adopted two situations, respectively, and achieve remarkably higher asymptotic performance compared to state-of-the-art human-guidance-based RLs. 

2) The proposed PHIL-TD3 achieves the most favorable performance, robustness, and adaptiveness in a series of metrics under the adopted two challenging autonomous driving tasks. 

3) The proposed \(TDQA\) mechanism prominently contributes to the advance of PHIL-TD3 and can well discriminate the quality of various human guidance to relieve humans by less requiring on human proficiency.

In future works, the proposed algorithm is expected to be transferred to a real-world ground vehicle, whereby the effect of human guidance on RL's optimization and control performance can be further examined in real life.

\bibliographystyle{IEEEtran}
\bibliography{reference}

\vfill

\appendix[]\label{appendix}
\newtheorem{theorem}{Theorem}
\newtheorem{remark}{Remark}
\newtheorem{proof}{Proof}

\begin{theorem}[Policy Optimality Invariance of the Human Intervention-based Reward Shaping] \label{therorem1}
Let the intervention-based reward shaping function \(F:S\times A\times S\rightarrow\mathbb{R}\) add a negative constant to the human intervened state as Eq. \ref{eq17}, if the human intervention will certainly occur at state \(\mathbf{s}_t\) when the next state \(\mathbf{s}_{t+1}\) is unacceptable, then the reward shaping function \(F\) does not change the policy optimality.
\end{theorem}

\begin{proof} 
According to \cite{RN30}, potential-based reward shaping function \(F:S\times A\times S\) is proven to be the only form that can preserve policy optimality. Specifically, \(F\) is represented as:
\begin{equation}\label{eqA1}
    F(\mathbf{s}_t,\mathbf{a}_t,\mathbf{s}_{t+1})=\gamma\Phi(\mathbf{s}_{t+1)} )-\Phi(\mathbf{s}_t),
\end{equation}
where \(\Phi:S\rightarrow\mathbb{R}\) is called the potential function defined over the state space.

Thus, the proof converts to construct potential function \(\Phi\). 

Define the potential function \(\Phi\) as:
\begin{equation}\label{eqA2}
        \Phi(\mathbf{s}_t)= 
\begin{cases}
    \frac{r_{\text{pen}}}{\gamma},& \text{if } \mathbf{s}_t \text{ is unacceptable}\\
    0,              & \text{otherwise}.
\end{cases}
\end{equation}

Then, when humans intervene in the state \(\mathbf{s}_t\) (meaning \(\mathbf{s}_{t+1}\) is unacceptable), \(F\) becomes:
\begin{equation}\label{eqA3}
\begin{split}
    F(\mathbf{s}_t,\mathbf{a}_t,\mathbf{s}_{t+1})&=\gamma\Phi(\mathbf{s}_{t+1} )-\Phi(\mathbf{s}_t )\\
    &=\frac{r_{\text{pen}}}{\gamma}\cdot\gamma-0=r_{\text{pen}}.
\end{split}
\end{equation}

And when humans do not intervene the state, \(F\) becomes:
\begin{equation}\label{eqA4}
    F(\mathbf{s}_t,\mathbf{a}_t,\mathbf{s}_{t+1})=\gamma\Phi(\mathbf{s}_{t+1} )-\Phi(\mathbf{s}_t )=0-0=0.
\end{equation}

Lumping Eqs. \ref{eqA3} and \ref{eqA4}, \(F\) turns into the reward-shaping term of Eq. \ref{eq17}, shown as:
\begin{equation}\label{eqA5}
\begin{split}
    r_t^{\text{shape}}=&r_t+F(\mathbf{s}_t,\mathbf{a}_t,\mathbf{s}_{t+1})\\
    =&r_t+r_{\text{pen}}\cdot[(\Delta_t=\mathbf{I} )\land(\Delta_{t-1}=\mathbf{0} )],
\end{split}
\end{equation}
where \([(\Delta_t=\mathbf{I} )\land(\Delta_{t-1}=\mathbf{0} )]\) refers to the intervention event of the human.

Hence, we complete the proof.

\end{proof} 

\begin{remark}\label{remark1}
Theorem \ref{therorem1} is established on the below assumptions: humans are considered to owe invariant judgment on the environment state. In this manner, the \(\Phi\) can be seen as a stable function defined in the state space.
\end{remark}

\begin{remark} 
The assumption of Remark \ref{remark1} is hard to be maintained in practice. This is because 1) the varying mental and physical status of one specific human participant would affect its accurate judgment on the environment state; 2) the judgment on the environment will be varying across different human participants; 3) the state space in the context of deep networks (the image-based one in our manuscript) is intractable to be identified by humans accurately.
\end{remark}

\setcounter{table}{0}
\renewcommand{\thetable}{A\arabic{table}}

\begin{table}[tb]
\caption{Configuration of the experimental platform.}
\begin{center}
\begin{adjustbox}{max width=\linewidth}
\begin{tabular}{ ccc } 
 \hline
 Type & Description & Details \\ 
 \hline
 Workstation & Operation system & \textit{Ubuntu 20.04}\\ 
 Workstation & CPU + RAM & \textit{AMD Ryzen 3900X + 32GB} \\ 
 Workstation & GPU & \textit{NVIDIA RTX 2080S} \\
 Driving simulator & Scenario software & \textit{CARLA}\\
 Driving simulator & Steering wheel suit & \textit{Logitech G29} \\
 Driving simulator & Displays & \textit{Joint heads-up monitors\(\times\)3} \\
 Driving simulator & Other equipment & \textit{Driver seat suit}\\
 Robotic vehicle & Vehicle brand & \textit{Wheeled UGV-Hunter}\\
 Robotic vehicle & Size dimension & \textit{1000mm \(\times\) 740mm \(\times\) 400mm}\\
 Robotic vehicle & Communication type & \textit{ROS publisher-subscriber}\\
 Robotic vehicle & Calculation board & \textit{Xavier NX Dev Kit}\\
 Other & Programming & \textit{Python}\\
 Other & Neural network toolbox & \textit{Pytorch}\\
 \hline
\label{tableA1}
\end{tabular}
\end{adjustbox}
\end{center}
\end{table}

\begin{table}[tb]
\caption{Architecture and details of value neural network (critic)}
\begin{center}
\begin{tabular}{ cc } 
 \hline
 Parameter & Value \\ 
 \hline
 Input (state + action) shape & [80,45,2] + [1]\\
 Network convolution Filter feature & [6,16] (kernel size 6 \(\times\) 6)\\
 Network pooling feature & Maxpooling (Stride 2)\\
 Network fully connected layer feature & [256,128,64]\\
 \hline
\label{tableA2}
\end{tabular}
\end{center}
\end{table}

\begin{table}[tb]
\caption{Architecture and details of policy neural network (actor)}
\begin{center}
\begin{tabular}{ cc } 
 \hline
 Parameter & Value \\ 
 \hline
 Input (state) shape & [80,45,2]\\
 Network convolution Filter feature & [6,16] (kernel size 6 \(\times\) 6)\\
 Network pooling feature & Maxpooling (Stride 2)\\
 Network fully connected layer feature & [256,128,64]\\
 \hline
\label{tableA3}
\end{tabular}
\end{center}
\end{table}

\begin{table}[tb]
\caption{Architecture and details of DAgger-based human policy model}
\begin{center}
\begin{adjustbox}{max width=\linewidth}
\begin{tabular}{ cc } 
 \hline
 Parameter & Value \\ 
 \hline
 Input (state) shape & [80,45,1] \\
 Network convolution Filter feature & [6,16] (kernel size 6 \(\times\) 6)\\
 Network pooling feature & Maxpooling (Stride 2)\\
 Network fully connected layer feature & [256,128,64]\\
 \hline
\label{tableA4}
\end{tabular}
\end{adjustbox}
\end{center}
\end{table}

\begin{table}[tb]
\caption{Hyperparameters for RL training}
\begin{center}
\begin{adjustbox}{max width=\linewidth}
\begin{tabular}{ ccc } 
 \hline
 Type & Description & Details \\ 
 \hline
 Maximum episode & Cutoff episode number of the training process & 400\\ 
 Minibatch size (\(N\)) & Capacity of minibatch & 128 \\ 
 Actor learning rate & Initial learning rate (policy/actor networks) & 5e-4 \\
 Critic learning rate & Initial learning rate (value/critic networks) & 2e-4\\
 Learning rate decay & Delay of learning rate (per episode) & 0.996 \\
 Activation function & Activation function of the networks & Relu \\
 Initial exploration & Initial exploration rate of noise in \(\epsilon\) greedy & 1\\
 Final exploration & Cutoff exploration rate of noise in \(\epsilon\) greedy & 0.05\\
 Gamma (\(\gamma\)) & Discount factor of the Bellman equation & 0.95\\
 Soft updating factor & Parameter update frequency to target networks & 1e-3\\
 Noise scale (\(\epsilon\)) & Noise amplitude of action in TD3 & 0.2\\
 Bounding box (\(c\)) & Bounding of the exploratory action in TD3 & 1\\
 Policy delay (\(d\)) & Update frequency of critic over actor & 1\\
 \hline
\label{tableA5}
\end{tabular}
\end{adjustbox}
\end{center}
\end{table}

\begin{table}[tb]
\caption{Hyperparameters for the PER mechanism}
\begin{center}
\begin{adjustbox}{max width=\linewidth}
\begin{tabular}{ ccc } 
 \hline
 Type & Description & Details \\ 
 \hline
 Replay buffer size & Capacity of PER buffer & 1e5\\
 Priority factor (\(\alpha\)) & Priority scaling factor & 0.6 \\ 
 Sample factor (\(\beta\)) & Importance sampling correlation & 1 \\
 Offset factor (\(\varepsilon\)) & Tiny constant avoiding zero retrieving probability & 1e-3\\

 \hline
\label{tableA6}
\end{tabular}
\end{adjustbox}
\end{center}
\end{table}

\begin{table}[tb]
\caption[Hyperparameters for the human policy model.]{Hyperparameters for DAgger-based human policy model.}
\begin{center}
\begin{adjustbox}{max width=\linewidth}
\begin{tabular}{ ccc } 
 \hline
 Type & Description & Details \\ 
 \hline
 Learning rate & Initial learning rate & 1e-4\\
 Activation function & Activation function of the network & Relu \\ 
 Maximum episode & Cutoff episode number of the training process & 50 \\
 Batch size & Capacity of minibatch & 128\\

 \hline
\label{tableA7}
\end{tabular}
\end{adjustbox}
\end{center}
\end{table}

\setcounter{figure}{0}
\renewcommand{\thetable}{A\arabic{figure}}

\begin{figure}[tb]
\begin{center}
\noindent
  \includegraphics[width=\linewidth]{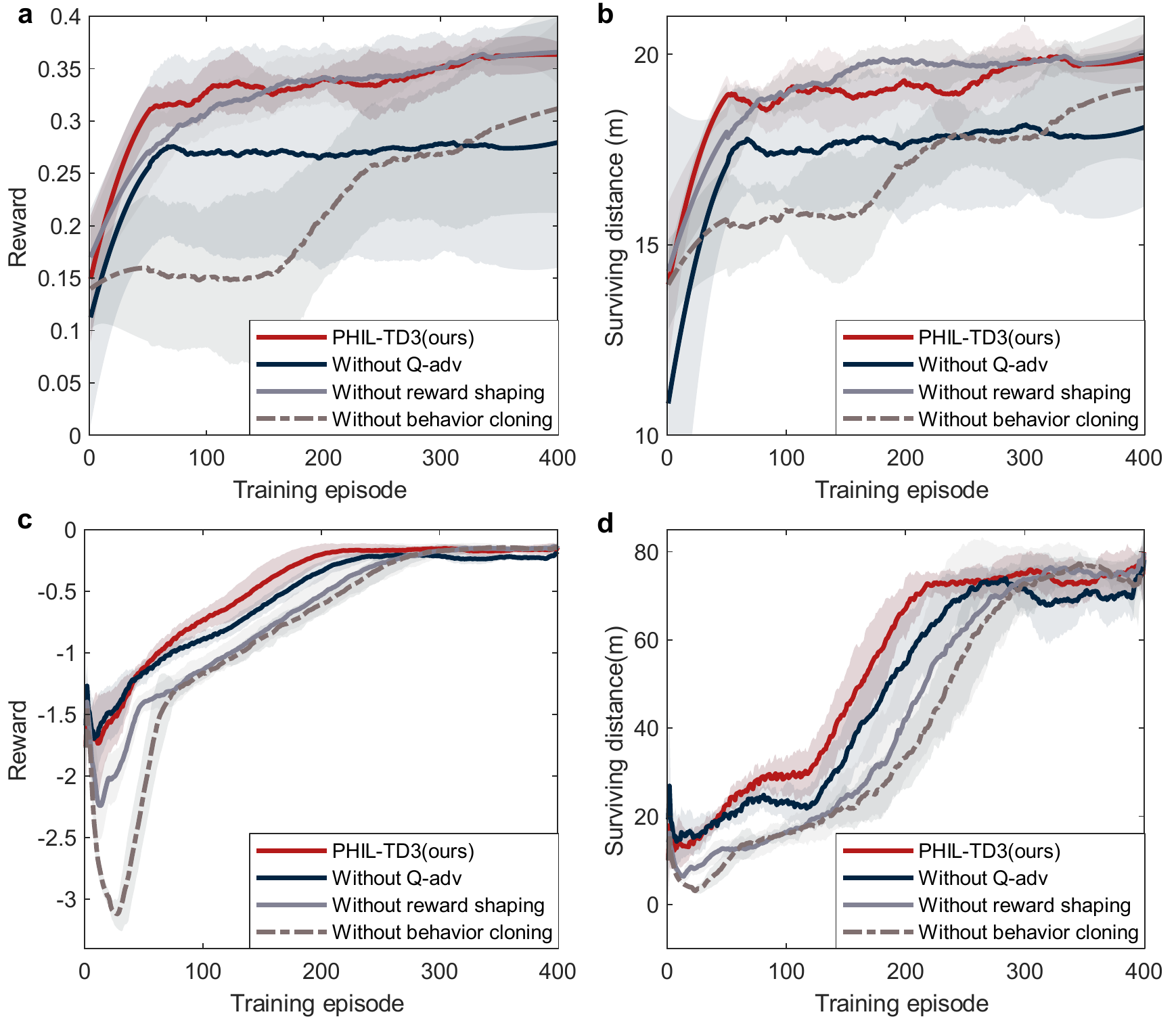}
  \end{center}
    \caption[Ablation study of the learning efforts of the proposed algorithm]{Ablation study of the proposed algorithm. a-b, curves of training rewards and surviving distances in the left-turn scenario, respectively. c-d, curves of training rewards and surviving distances in the congestion scenario, respectively. }\label{Figures1}
\end{figure}

\end{document}